\documentclass{article}

\usepackage{microtype}
\usepackage{graphicx}
\usepackage{subcaption}
\usepackage{booktabs} 
\usepackage{subcaption}
\usepackage{overpic}
\usepackage{pgfplots}
\usepackage{multirow}
\usepackage{tikz}
\usepackage{pgfplots}

\usepackage{hyperref}

\usepackage[preprint]{icml2026}

\usepackage{amsmath}
\usepackage{amssymb}
\usepackage{mathtools}
\usepackage{amsthm}

\usepackage[capitalize,noabbrev]{cleveref}

\theoremstyle{plain}
\newtheorem{theorem}{Theorem}[section]

\newtheorem{lemma}[theorem]{Lemma}

\theoremstyle{definition}

\theoremstyle{remark}

\usepackage[textsize=tiny]{todonotes}

\icmltitlerunning{Submission and Formatting Instructions for ICML 2026}

\begin{document}

\twocolumn[
  \icmltitle{From Basis to Basis: Gaussian Particle Representation \\ for Interpretable PDE Operators}



  \icmlsetsymbol{equal}{*}

  \begin{icmlauthorlist}
    \icmlauthor{Zhihao Li}{hkust_gz}
    \icmlauthor{Yu Feng}{hkust_gz}
    \icmlauthor{Zhilu Lai}{hkust_gz,hkust}
    \icmlauthor{Wei Wang}{hkust_gz,hkust}
  \end{icmlauthorlist}

  \icmlaffiliation{hkust_gz}{The Hong Kong University of Science and Technology (Guangzhou), Guangzhou, China}
  \icmlaffiliation{hkust}{The Hong Kong University of Science and Technology, Hong Kong SAR, China}

  \icmlcorrespondingauthor{Zhihao Li}{zli416@connect.hkust-gz.edu.cn}
  \icmlcorrespondingauthor{Wei Wang}{weiwcs@ust.hk}

  \icmlkeywords{Machine Learning, ICML}

  \vskip 0.3in
]



\printAffiliationsAndNotice{}  


\begin{abstract}
Learning PDE dynamics for fluids increasingly relies on neural operators and Transformer-based models, yet these approaches often lack interpretability and struggle with localized, high-frequency structures while incurring quadratic cost in spatial samples. We propose to represent fields with a \emph{Gaussian basis}, where learned atoms carry explicit geometry (centers, anisotropic scales, weights) and form a compact, mesh-agnostic, directly visualizable state. Building on this representation, we introduce a \emph{Gaussian Particle Operator} that acts \emph{in modal space}: learned \emph{Gaussian modal windows} perform a Petrov--Galerkin measurement, a \emph{PG Gaussian Attention} effects global cross-scale coupling. This basis-to-basis design is resolution-agnostic and achieves near-linear complexity in $N$ for fixed modal budget, supporting irregular geometries and seamless 2D$\to$3D extension. On standard PDE benchmarks and real datasets, our method attains state-of-the-art–competitive accuracy while providing intrinsic interpretability.
\end{abstract}

\section{Introduction}

Fluid-governed PDEs \citep{02:pde,82:intro} underpin critical real-world systems, from numerical weather prediction and climate reanalysis to ocean circulation and engineering aerodynamics \citep{20:turb,87:turb}. Classical solvers (finite element/volume and spectral methods) \citep{02:pde,12:fem,84:diff} deliver high fidelity but face persistent challenges: strongly multi-scale dynamics, mesh dependence and complex geometries, stiffness in time integration, and high computational cost for long rollouts. Neural operators \citep{21:fno,23:no,21:DeepONet} emerged as data-driven maps between function spaces, enabling resolution-agnostic surrogates; more recently, Transformer-based operators \citep{21:GT, 23:GNOT} leverage attention to capture long-range interactions and achieve strong empirical performance on diverse PDE tasks. However, these models still suffer from two key limitations: \emph{(i) poor interpretability}—latent features and attention weights are rarely tied to physically meaningful modes; and \emph{(ii) localization/frequency bias}—global self-attention tends to favor low-rank, low-frequency correlations, making sharp fronts, vortical filaments, and other high-frequency structures harder to model, while na\"ively scaling attention over $N$ spatial samples incurs $\mathcal{O}(N^2)$ cost \citep{25:HarnessingSP}.

We advocate representing fluid fields with a \emph{Gaussian (particle) basis} rather than a fixed grid, hand-picked spectra \citep{21:mwt,24:M2NO}, or a monolithic implicit network \citep{23:OperatorLN, 24:AROMA}. Gaussian atoms carry \emph{explicit geometry}—centers and (anisotropic) scales—which align naturally with coherent flow structures (vortices, filaments, fronts), afford multi\-scale locality, and are directly \emph{visualizable} and \emph{differentiable}. This basis is mesh\-agnostic and compact, supports irregular boundaries, and extends seamlessly from 2D to 3D \citep{00:RadialBF,91:UniversalAU}. While prior neural representations often rely on global Fourier features, wavelets, or black\-box INRs, \emph{learning a particleized Gaussian basis as the primary state of the field} has been scarcely explored and offers a clearer bridge to physical intuition. Concretely, a field is approximated by weighted Gaussians with $\mu_i$ (centers), $\sigma_i$ (scales, possibly anisotropic), and $w_i$ (mixture weights) learned from data; evaluating these atoms at query locations yields a compact coefficient vector that serves as the field’s interpretable latent state (basis).

\begin{figure*}[t]
    \centering
    \includegraphics[width=\textwidth]{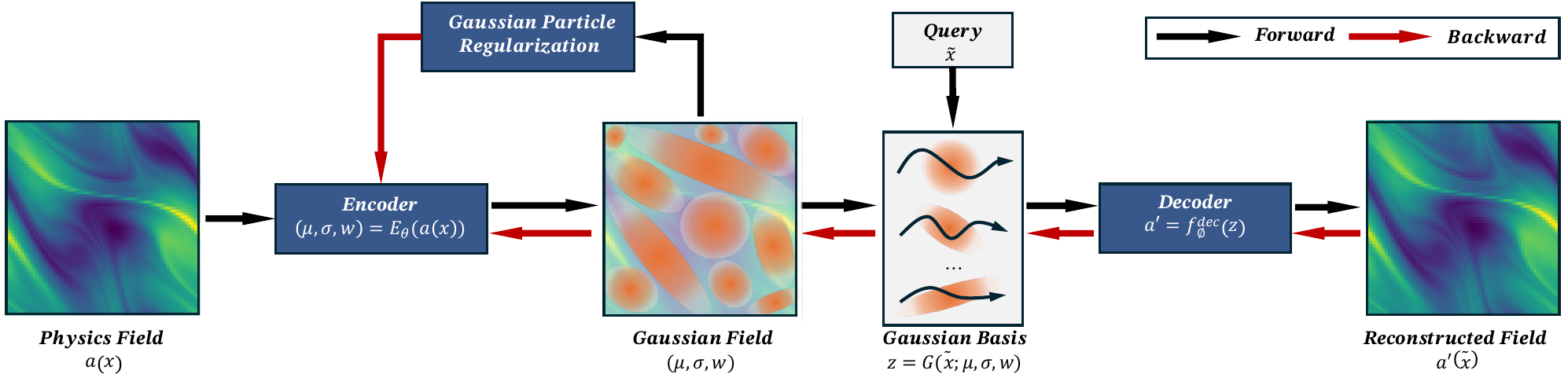}
    \caption{\textbf{Overview of the Gaussian Basis Field framework.} Given the physics field $a(x)$, an encoder $E_\theta$ produces $G$ Gaussian components per spatial location (mean $\mu$, scale $\sigma$, and mixture weight $w$). These define a Gaussian field that is evaluated at queries to form the basis $z$, which is decoded by a fixed decoder $f_\phi^{\mathrm{dec}}$ to reconstruct the output field.}
    \label{fig:gaussian_field}
\end{figure*}

We present an interpretable, resolution-agnostic neural operator that learns a \emph{Gaussian particle} basis for fields and couples it with a \emph{Petrov--Galerkin Gaussian Attention} layer, enabling basis-to-basis modeling with near-linear complexity and strong accuracy on 2D/3D and irregular domains. Our contributions are summarized as follows:

(1) \textbf{Gaussian Particle Representation.} An encoder learns per-site Gaussians $(\mu,\sigma,w)$; evaluating at arbitrary queries yields an interpretable, visualizable basis $Z$ that is mesh-agnostic and extends seamlessly to 3D.

(2) \textbf{PG Gaussian Attention.} Learned \emph{Gaussian modal windows} perform PG-style measurement ($N\!\to\!G$), a $G{\times}G$ attention implements the modal kernel (global coupling), and the result is scattered back ($G\!\to\!N$), yielding a principled and interpretable operator.

(3) \textbf{Efficiency \& Scalability.} With a small modal budget $G\!\ll\!N$, spatial transfers scale $\mathcal{O}(N)$ and modal attention is independent of $N$, delivering near-linear growth with resolution and supporting multi-step operator stacking.

(4) \textbf{Empirical validation.} Across standard PDE benchmarks and real datasets (including ERA5 and 3D/irregular domains), our approach attains state-of-the-art–competitive accuracy while providing \emph{intrinsic interpretability} (particle and modal diagnostics), yielding improved spectral fidelity and rollout stability—demonstrating a favorable accuracy–interpretability trade-off.

\section{Methodology}
\begin{figure*}
    \centering
    \includegraphics[width=\linewidth]{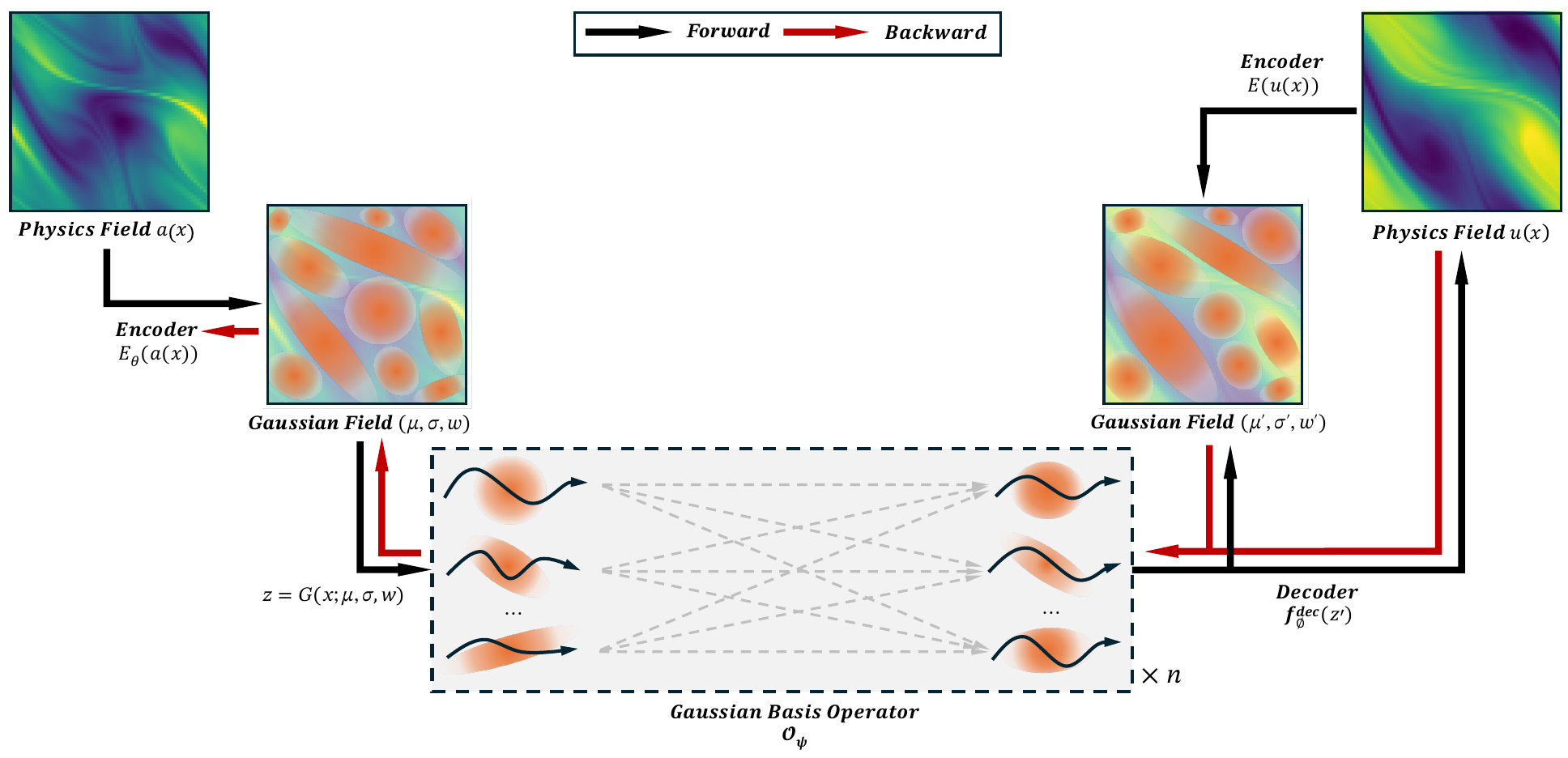}
    \vspace{-10pt}
    \caption{\textbf{Architecture of the Gaussian Particle Operator (GPO).} The pipeline encodes $a(\mathbf{x})$ into a Gaussian field $(\mu,\sigma,w)$, evaluates a basis $Z$, applies the modal operator $\mathcal{O}_{\psi}$, and decodes to $\hat{u}(\mathbf{x})$; black arrows denote forward computation and red arrows denote gradients.}
    \label{fig:framework}
\end{figure*}
\subsection{Gaussian Basis Representation}
\label{sec:gaussian_basis}

\paragraph{From physics field to Gaussian field and basis.}
We represent a spatial field $a:\Omega\!\subset\!\mathbb{R}^{d}\!\to\!\mathbb{R}^{d_a}$ by a set of \emph{Gaussian particles} placed at each sample location $\mathbf{x}_j$. Each particle is parameterized by a center $\mu_{j,i}\!\in\!\mathbb{R}^{d}$, axis-aligned scale $\sigma_{j,i}\!\in\!\mathbb{R}^{d}$, and mixture weight $w_{j,i}\!\in\![0,1]$ with $\sum_{i=1}^{G}w_{j,i}=1$. The associated (unnormalized) kernel is
\begin{equation}
G(\tilde{\mathbf{x}};\mu_{j,i},\sigma_{j,i})
=\exp\!\Big(-\tfrac12\big\|(\tilde{\mathbf{x}}-\mu_{j,i})/\sigma_{j,i}\big\|_2^2\Big).
\end{equation}
Evaluating these kernels at query $\tilde{\mathbf{x}}_{j}$ yields the \emph{Gaussian basis} coefficients
\begin{equation}
z_{j,i}=w_{j,i}\,G(\tilde{\mathbf{x}}_{j};\mu_{j,i},\sigma_{j,i}),
\label{eq:weighted-eval}
\end{equation}
where $\mathbf{z}_j=[z_{j,1},\ldots,z_{j,G}]^\top\in\mathbb{R}^{G}$. Physically, the Gaussian field acts as a mollified, locally supported expansion of $a(\cdot)$; the coefficients in Eq.(\ref{eq:weighted-eval}) can be viewed as localized averages of $a$ under data-adaptive windows $(\mu,\sigma)$, while $w$ distributes mass among overlapping particles. This basis is resolution-agnostic and naturally extends to irregular geometries.

\paragraph{Encoder.}
Given samples $\{(x_j,a_j)\}_{j=1}^{N}$, the encoder $E_{\theta}$ injects geometry by
embedding coordinates and fusing them with field features. Concretely, we form
$\eta_j=[\,a_j,\;\gamma(x_j)\,]$, where $\gamma(\cdot)$ is a fixed Fourier-feature map (or a small MLP),
and compute
\begin{align}
\phi_j &= \mathrm{ReLU}(W_{\mathrm{in}}\eta_j),\\
\mu_j &= W_\mu\,\mathrm{ReLU}(U_\mu \phi_j),\\
\sigma_j &= \mathrm{Softplus}\!\big(W_\sigma\,\mathrm{ReLU}(U_\sigma \phi_j)\big),\\
w_j &= \mathrm{Softmax}\!\big(W_w\,\mathrm{ReLU}(U_w \phi_j)\big),
\end{align}
reshaped as $\mu_{j,i},\sigma_{j,i}\in\mathbb{R}^{d}$ and $w_{j,i}\in[0,1]$ with $\sum_i w_{j,i}=1$. We use $\gamma(x)$ as a fixed positional encoding (Fourier features):
$\gamma(x)=[\sin(2\pi Bx),\cos(2\pi Bx)]$, where $B\in\mathbb{R}^{m\times d}$ is sampled once from
$\mathcal{N}(0,\sigma_B^2)$ and kept fixed; the resulting coordinate embedding has dimension $2m$.

\paragraph{Gaussian basis evaluation.}
With $(\mu,\sigma,w)$ predicted by $E_{\theta}$, the weighted Gaussian evaluation Eq.(\ref{eq:weighted-eval}) produces the per-site latent vector $\mathbf{z}_j$. \emph{Physically}, $\mu$ encodes particle locations, $\sigma$ controls receptive-field sizes (anisotropy along axes), and $w$ balances overlapping contributions. \emph{Computationally}, the map $(\mu,\sigma,w,\tilde{\mathbf{x}})\mapsto \mathbf{z}$ is local and embarrassingly parallel.

\paragraph{Decoder.}
A lightweight MLP head $f^{\mathrm{dec}}_{\phi}:\mathbb{R}^{G}\!\to\!\mathbb{R}^{c_{\mathrm{out}}}$ regresses from $\mathbf{z}_j$ to the field value at the query:
\begin{equation}
\hat a(\tilde{\mathbf{x}}_{j})=f^{\mathrm{dec}}_{\phi}(\mathbf{z}_j).
\label{eq:decoder}
\end{equation}
In practice, we use a two-layer perceptron with ReLU.

\paragraph{Gaussian particle regularization.}
Because the constraints act on the \emph{particle parameters} produced by the encoder, we impose them at the Gaussian-field level (conceptually tied to $E_{\theta}$ but applied after parameter prediction):
\begin{align}
\mathcal{L}_{\mu}&=\frac{1}{N}\sum_{j=1}^{N}\Big\|\sum_{i=1}^{G} w_{j,i}\mu_{j,i}-\mathbf{x}_j\Big\|_2^{2},\\
\mathcal{L}_{\sigma} &= \frac{1}{NGd}\sum_{j,i,\ell}\Big[\,[\sigma_{j,i,\ell}-\sigma_{\max}]_{+}+[\sigma_{\min}-\sigma_{j,i,\ell}]_{+}\Big],
\end{align}
which promote spatial interpretability (centers near coordinates), avoid degenerate particles, and discourage overly peaky mixtures.

\paragraph{Overview pipeline.}
Eqs.(~\ref{eq:weighted-eval}–\ref{eq:decoder}) define the Gaussian-field training pipeline, and the complete forward/backward diagram is in Figure~\ref{fig:gaussian_field}. We minimize the reconstruction loss together with $\mathcal{L}_{\mu}$, $\mathcal{L}_{\sigma}$. 

\paragraph{Approximation capacity of the Gaussian basis.}
We record a standard density result:

\begin{lemma}[Density of Gaussian mixtures]
\label{lem:gm-density-main}
Let $\Omega\!\subset\!\mathbb{R}^{d}$ be compact. Finite \emph{linear combinations} of (possibly anisotropic) Gaussian kernels are dense in $C(\Omega)$
and in $L^{r}(\Omega)$ for $1\le r<\infty$. Consequently, for any continuous scalar field $v$ and $\varepsilon>0$,
there exist $G$ and parameters $\{(\mu_i,\sigma_i,c_i)\}_{i=1}^{G}$ with $c_i\in\mathbb{R}$ such that
\[
\Big\|v(\cdot)-\sum_{i=1}^{G} c_i\,
\exp\!\Big(-\tfrac{1}{2}\|(\cdot-\mu_i)/\sigma_i\|_2^{2}\Big)\Big\|_{\infty}<\varepsilon.
\]
Vector-valued fields admit componentwise approximation. \emph{(Proof in Appx.\ref{sec:exp_gf})}
\end{lemma}

\subsection{Petrov--Galerkin Gaussian Attention}
\label{sec:pg-attn}

Sec~\ref{sec:gaussian_basis} defines \emph{local} Gaussian bases: at each location $j$, we obtain $G$ weighted coefficients $\mathbf{z}_j\in\mathbb{R}^{G}$ from particles $(\mu,\sigma,w)$. To learn a resolution-agnostic operator, we adopt a Petrov--Galerkin (PG) view—expanding in a trial space and testing with a (possibly different) test space—and move from the spatial grid to \emph{modal space} \citep{06:RevisitingSF,82:StreamlinePG}. Learned Gaussian modal windows first \emph{measure} the field by pooling across locations, a global \emph{mode coupling} acts on the $G$ components, and the result is \emph{scattered back} to locations. This PG pipeline yields an efficient and interpretable basis-to-basis operator.

\subsubsection{From Petrov--Galerkin projection to a Gaussian-basis operator}
\label{sec:pg-link}

In Petrov--Galerkin, we expand the field in a \emph{trial} space and evaluate it with a \emph{test} space. Here, Gaussian particles form the trials, while \emph{Gaussian modal windows} serve as discrete tests that pool local information into global modes.

\paragraph{Trial functions (Gaussian basis).}
Let the unnormalized Gaussian particle (anchored at location $j$, component $i$) be
\begin{equation}
\phi_{j,i}(\mathbf{x})
=\exp\!\Big(-\tfrac{1}{2}\big\|(\mathbf{x}-\mu_{j,i})/\sigma_{j,i}\big\|_2^{2}\Big),
\Sigma_{j,i}=\mathrm{diag}(\sigma_{j,i}^2).
\label{eq:pg-trial}
\end{equation}
With weighted evaluations (Sec.~\ref{sec:gaussian_basis}), each site provides
$\mathbf{z}_j=[z_{j,1},\ldots,z_{j,G}]^\top\in\mathbb{R}^{G}$, where
$z_{j,i}=w_{j,i}\,\phi_{j,i}(\tilde{\mathbf{x}}_j)$.

\paragraph{Test functions (Gaussian modal windows).}
We distinguish two normalizations: (i) row-normalized assignments $\sum_{g}p_{j,g}=1$ (used to map each
location to modes), and (ii) column-normalized weights $\bar p_{j,g}=p_{j,g}/\sum_{j'}p_{j',g}$ so that
$\sum_{j}\bar p_{j,g}=1$ (used for PG measurement).
\begin{equation}
\psi_g(x)\approx \sum_{j=1}^{N}\bar p_{j,g}\,\delta(x-\tilde x_j),
\qquad \bar p_{j,g}=\frac{p_{j,g}}{\sum_{j'}p_{j',g}}.
\end{equation}
Using a linear projection of coefficients $\mathbf{s}_j=\mathbf{z}_j W_z\in\mathbb{R}^{D}$,
the PG \emph{measurement} (test of the trial field) yields modal tokens
\begin{equation}
t_g=\sum_{j=1}^{N}\bar p_{j,g}\,s_j,\qquad s_j=z_j W_z.
\label{eq:pg-measure}
\end{equation}

\paragraph{Modal coupling and scatter.}
Let $\kappa:\{1,\ldots,G\}^2\to\mathbb{R}^{D\times D}$ be a (learned) coupling kernel over modes.
PG updates the modal state and scatters it back:
\begin{align}
    U_g=\sum_{g'=1}^{G} \kappa(g,g')\,t_{g'}\in\mathbb{R}^{D},\\
    \widetilde{\mathbf{z}}_j=\Big(\sum_{g=1}^{G} p_{j,g}\,U_g\Big) W_{\mathrm{out}}\in\mathbb{R}^{G}.
\label{eq:pg-couple-scatter}
\end{align}
Stacking sites gives $\widetilde Z\in\mathbb{R}^{N\times G}$. Algebraically,
\begin{equation}
\boxed{\;\;\widetilde Z
\;\approx\;
A\,\mathcal{K}\,A^{\top}\,(Z\,W_{z})\,W_{\mathrm{out}},}
\label{eq:pg-lowrank}
\end{equation}
where $A[j,g]=p_{j,g}$ and $\mathcal{K}[g,g']$ encodes modal coupling. 
Thus PG supplies the \emph{structure}: \(\)test (measure) $\rightarrow$ couple $\rightarrow$ scatter.

\begin{table*}[t] 
\caption{\textbf{Performance comparison with baselines on benchmarks.} $L_{2}$ loss is recorded.} 
\label{table:all} 
\begin{sc} 
\renewcommand{\multirowsetup}{\centering} 
\footnotesize 
\resizebox{\linewidth}{!}{ 
\begin{tabular}{c|cc|ccc|c|cc} 
\toprule 
Model & NS2D & NS3D & ERA5-temp & ERA5-wind u & Carra & Airfoil & Turbulent & PlanetSWE \\ 
\midrule 
(Geo-)FNO & 3.24E-02 & 5.07E-01 & 7.09E-03 & 1.02E-01 & 3.50E-01 & 5.12E-03 & 4.82E-01 & 9.23E-02 \\ 
M2NO      & 2.76E-02 & 4.65E-01 & \underline{3.35E-03} & \underline{7.15E-02} & 3.62E-01 & 1.98E-03 & 4.53E-01 & 8.96E-02 \\
CNO       & 9.27E-02 & 8.12E-01 & 1.24E-02 & 1.86E-01 & 6.89E-01 & 4.65E-03 & 7.91E-01 & 1.57E-01 \\
LSM       & \underline{3.11E-02} & 3.80E-01 & 5.86E-03 & 8.23E-02 & 4.05E-01 & 5.37E-03 & \underline{4.16E-01} & 9.58E-02 \\
AMG       & 3.12E-02 & 4.32E-01 & 6.11E-03 & 8.79E-02 & \underline{3.21E-01} & 3.09E-03 & 4.47E-01 & \underline{8.14E-02} \\
\midrule
GT        & 8.81E-02 & 5.39E-01 & 5.44E-03 & 1.55E-01 & 3.73E-01 & 8.26E-03 & 5.41E-01 & 1.23E-01 \\
GNOT      & 7.19E-01 & 1.01E+00 & 1.55E-02 & 3.49E-01 & 7.57E-01 & 1.53E-02 & 9.86E-01 & 2.14E-01 \\
Transolver& 3.76E-02 & 5.29E-01 & 4.18E-03 & 1.06E-01 & 3.76E-01 & 9.46E-03 & 5.33E-01 & 1.05E-01 \\
ONO       & 4.26E-02 & 8.83E-01 & 1.45E-02 & 3.49E-01 & 7.25E-01 & 5.26E-03 & 8.27E-01 & 1.69E-01 \\
LNO       & 4.81E-02 & \underline{3.68E-01} & 7.32E-03 & 1.31E-01 & 4.36E-01 & \underline{4.39E-03} & 5.69E-01 & 9.12E-02 \\
\midrule
\textbf{GPO} & \textbf{3.02E-02} & \textbf{3.44E-01} & \textbf{2.26E-03} & \textbf{6.68E-02} & \textbf{2.97E-01} & \textbf{1.12E-03} & \textbf{3.95E-01} & \textbf{8.01E-02} \\
\bottomrule 
\end{tabular}} 
\end{sc} 
\end{table*}

\subsubsection{Attention as a parameterization of the PG operator}
\label{sec:attn-param}

We now instantiate Eq.(~\ref{eq:pg-lowrank}) with a multi-head attention layer that is global in
\emph{modal} space and local in the $N\!\leftrightarrow\!G$ transfers. Let
$Z\in\mathbb{R}^{N\times G}$ and particle parameters
$(\mu,\sigma,w)\in\mathbb{R}^{N\times G\times d}\times\mathbb{R}^{N\times G\times d}\times\mathbb{R}^{N\times G}$.

\paragraph{Learned Gaussian modal windows.}
For head $h$, form a per-site descriptor $\xi_j=[\,\mathbf{z}_j,\mathbf{w}_j,\mu_j,\sigma_j\,]\in\mathbb{R}^{G(2d+2)}$ and project to $h_j^{(h)}\!\in\!\mathbb{R}^{D}$. A softmax over modes produces windows
\begin{equation}
p^{(h)}_{j,g}=\mathrm{softmax}_{g}\!\big(W_{p}^{(h)}h^{(h)}_j\big),
\label{eq:attn-win}
\end{equation}
which instantiate the PG test functions in discrete form. $W_{p}^{(h)}\!\in\!\mathbb{R}^{G\times D}$ is the (head-specific) linear projection that maps the local embedding
$h^{(h)}_j$ at location $j$ to mode logits over the $G$ Gaussian modes.

\paragraph{PG measurement $N\!\to\!G$.}
Project coefficients $\mathbf{s}^{(h)}_j=\mathbf{z}_j W^{(h)}_z$ and compute tokens
\begin{equation}
t^{(h)}_g=\frac{\sum_{j} p^{(h)}_{j,g}\,\mathbf{s}^{(h)}_j}{\sum_{j} p^{(h)}_{j,g}}
\in\mathbb{R}^{D},
T^{(h)}=[t_1^{h},\ldots,t_G^{h}]\in\mathbb{R}^{G\times D},
\label{eq:attn-agg}
\end{equation}
which matches the PG measurement in Eq.(~\ref{eq:pg-measure}).

\paragraph{Global modal coupling ($G{\times}G$ attention).}
Scaled dot-product attention parameterizes the kernel $\mathcal{K}$:
\begin{align}
Q^{(h)} &= T^{(h)} W^{(h)}_{Q},
K^{(h)} = T^{(h)} W^{(h)}_{K}, 
V^{(h)} = T^{(h)} W^{(h)}_{V}, \label{eq:attn-qkv}\\
\alpha^{(h)} &= \mathrm{softmax}\!\left(\frac{Q^{(h)} {K^{(h)}}^{\!\top}}{\sqrt{D}}\right), \label{eq:attn-softmax}\\
\widetilde{T}^{(h)} &= \alpha^{(h)} V^{(h)} \;\in\; \mathbb{R}^{G\times D}. \label{eq:attn-out}
\end{align}
Here $\alpha^{(h)}(g,g')$ plays the role of a data-driven modal coupling $\kappa(g,g')$.

\paragraph{Scatter $G\!\to\!N$ and readout.}
Using the same windows, scatter the coupled modes back and read out to $G$ coefficients:
\begin{align}
    y^{(h)}_j&=\sum_{g} p^{(h)}_{j,g}\,U^{(h)}_{g}\in\mathbb{R}^{D},\\
\widetilde{\mathbf{z}}_j&=\Big(\big\|_{h=1}^{H} y^{(h)}_j\Big) W_{\mathrm{out}}\in\mathbb{R}^{G},\\
\widetilde Z&=[\widetilde{\mathbf{z}}_1^\top;\ldots;\widetilde{\mathbf{z}}_N^\top].
\label{eq:attn-scatter}
\end{align}

To stabilize training and preserve the per-site total mass (row-wise $\ell_1$ sum), we first take a
convex residual update with a mixing coefficient $\lambda\!\in\![0,1]$ and then renormalize each row:
\begin{align}
\widehat Z &= (1-\lambda)\,Z \;+\; \lambda\,\widetilde Z, \label{eq:residual}\\
Z'_{j,:} &= \frac{\sum_{g=1}^{G} Z_{j,g}}{\sum_{g=1}^{G} \widehat Z_{j,g} + \varepsilon}\;\widehat Z_{j,:}\,,
\qquad j=1,\ldots,N, \label{eq:row-renorm}
\end{align}
where $\varepsilon\!>\!0$ avoids division by zero. Eq.(~\ref{eq:residual}) provides a conservative
blend between the old and updated coefficients, while Eq.(~\ref{eq:row-renorm}) rescales each site’s
coefficients so that $\sum_{g} Z'_{j,g}=\sum_{g} Z_{j,g}$. 

\noindent\emph{Complexity.}
Per head, forming the window logits requires projecting $\xi_j\in\mathbb{R}^{G(2d+2)}$ to $\mathbb{R}^{D}$,
costing $\mathcal{O}(BHN\cdot G(2d+2)\cdot D)$. The two transfers $N\!\leftrightarrow\!G$ cost
$\mathcal{O}(BHNGD)$, and modal attention costs $\mathcal{O}(BHG^2D)$ (independent of $N$).
Overall, for fixed $G$ the dominant term is linear in $N$, with explicit dependence on spatial dimension $d$.

\paragraph{Expressivity of the modal operator.}
We formalize that PG Gaussian Attention can approximate a broad class of continuous operators:

\begin{theorem}[Universal approximation in modal form]
\label{thm:universality-main}
Let $\mathcal{T}:L^{p}(\Omega;\mathbb{R}^{c_{\mathrm{in}}})\!\to\!L^{q}(\Omega;\mathbb{R}^{c_{\mathrm{out}}})$
be continuous on bounded sets and admit either a Mercer/Hilbert--Schmidt kernel
$K(\mathbf{x},\mathbf{x}')$ or a low-rank factorization
$\mathcal{T}\!\approx\!\Phi(\cdot)\,\mathcal{K}\,\Phi(\cdot)^{\!\top}$ with continuous
$\Phi:\Omega\!\to\!\mathbb{R}^{m}$. Then, for any $\varepsilon>0$, there exist a modal budget $G$
and parameters $\Theta$ of our encoder, Gaussian modal windows, PG Gaussian Attention, and decoder
such that $\|\mathcal{G}_{\Theta}-\mathcal{T}\|<\varepsilon$ (operator norm on bounded subsets).
\emph{(Proof in Appx.\ref{sec:exp_gpo})}
\end{theorem}

\subsection{Gaussian Particle Operator: Overall Framework}
\label{sec:overall}

\subsubsection{Neural Operator Formulation}
Let $\Omega\!\subset\!\mathbb{R}^{d}$ be the domain, $a:\Omega\!\to\!\mathbb{R}^{c_{\text{in}}}$ the input
field, and $u:\Omega\!\to\!\mathbb{R}^{c_{\text{out}}}$ the target field. We model the map
$a\mapsto u$ by a neural operator
\begin{equation}
\label{eq:op-def}
\mathcal{G}_{\Theta}\;=\; f^{\mathrm{dec}}_{\phi}\ \circ\ 
\Big(\mathcal{O}_{\psi}\Big)^{\!\circ n}\ \circ\ 
\mathcal{Z}\big(\,\cdot\;;\,\Pi_\theta(\cdot)\big)\ \circ\ E_{\theta},
\end{equation}
where:
\begin{itemize}\setlength{\itemsep}{2pt}
\item $E_{\theta}$ (encoder) predicts Gaussian particles
$\Pi_{\theta}(a)=\big(\mu_{\theta},\sigma_{\theta},w_{\theta}\big)$ on the \emph{context} locations $\{x_j\}_{j=1}^{N}$
(using both $a_j$ and $x_j$);
\item $\mathcal{Z}(\cdot;\Pi)$ evaluates the \emph{Gaussian basis} and returns per-location,
$G$-dimensional coefficients $Z\in\mathbb{R}^{N\times G}$ with
\begin{equation}
\label{eq:Z-eval}
z_{j,i}\;=\;w_{j,i}\,\exp\!\Big(-\tfrac{1}{2}\big\|(\mathbf{x}_j-\mu_{j,i})/\sigma_{j,i}\big\|_2^{2}\Big);
\end{equation}
\item $\mathcal{O}_{\psi}$ is the \emph{Gaussian-basis operator} (Sec.~\ref{sec:pg-attn}) acting on
$Z$ and parameterized by PG Gaussian Attention; it can be applied $n$ times:
\begin{align}
    \label{eq:O-iter}
&Z^{(0)}=Z,\\
&Z^{(k+1)}=\mathcal{O}_{\psi}\big(Z^{(k)};\Pi_{\theta}(a)\big), 
\end{align}
\item $f^{\mathrm{dec}}_{\phi}$ (decoder) maps the updated basis to the output field values:
$\ \hat{u}(\mathbf{x}_j)=f^{\mathrm{dec}}_{\phi}\!\big(Z^{(n)}_{j,:}\big)$.
\end{itemize}

The construction is resolution-agnostic: for any query set $\{\tilde x_m\}_{m=1}^{M}$ (2D/3D grids or irregular meshes),
we keep the particle field $\Pi_{\theta}(a)$ predicted on the context locations fixed and simply re-evaluate
$z_{j,i}(\tilde x_m)=w_{j,i}\,\exp\!\big(-\tfrac12\|(\tilde x_m-\mu_{j,i})/\sigma_{j,i}\|_2^2\big)$,
followed by the same $\mathcal{O}_{\psi}$ and $f^{\mathrm{dec}}_{\phi}$.

\subsubsection{Pipeline overview}
As shown in Fig.~\ref{fig:framework}, given an input field $a(\mathbf{x})$, the encoder $E_{\theta}$ produces per-site Gaussian particles $\Pi_{\theta}(a)=(\mu,\sigma,w)$, i.e., a \emph{Gaussian field}. At query locations $\{\mathbf{x}_j\}_{j=1}^{N}$, we then evaluate the Gaussian basis by Eq.(\ref{eq:Z-eval}) to obtain $Z\in\mathbb{R}^{N\times G}$. The Gaussian-basis operator $\mathcal{O}_{\psi}$ acts on $Z$ in modal space and can be applied for $n$ stages as in Eq.(\ref{eq:O-iter}) to capture multi-step coupling, yielding $Z^{(n)}$. Finally, the decoder $f^{\mathrm{dec}}_{\phi}$ maps $Z^{(n)}_{j,:}$ to $\hat{u}(\mathbf{x}_j)$. During training, the target $u(\mathbf{x})$ may also be encoded by $E_{\theta}$ to provide an auxiliary Gaussian-field supervision signal.

\section{Experiments}
\begin{figure*}[t]
    \centering
    \includegraphics[width=0.85\linewidth]{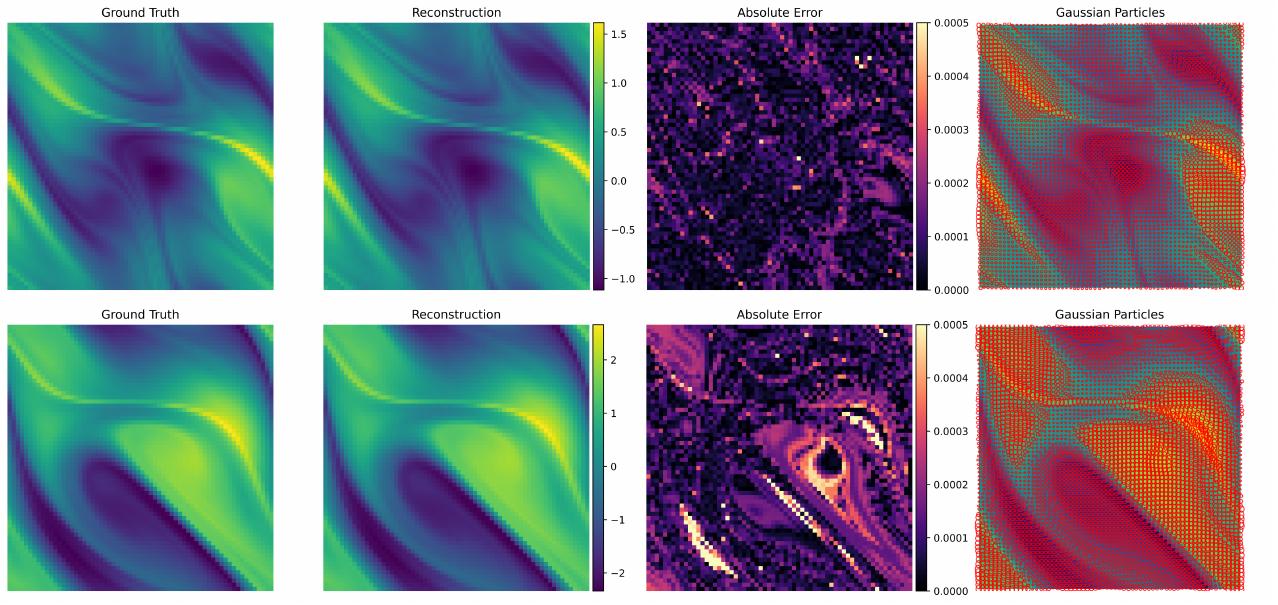}
    \caption{Interpretable visualization on an \textbf{in-distribution sample (above)} and an \textbf{out-of-distribution sample (below)}. Left to right: ground truth, reconstruction from the Gaussian basis, absolute error, and learned Gaussian particles overlaid (ellipses: center $\mu$, axes $\propto\sigma$, color/size $\propto w$). }
    \vspace{-5pt}
    \label{fig:recon}
\end{figure*}
\paragraph{Benchmarks.}
We evaluate on two Navier--Stokes surrogates, three simulated PDE systems, and two real reanalyses to span 2D$\rightarrow$3D, planar$\rightarrow$spherical, and regular$\rightarrow$irregular domains. \textbf{NS2D} \citep{23:no} is an incompressible periodic box sampled on $64{\times}64$; \textbf{NS3D} \citep{22:bench} extends to a periodic cube on $64^3$, stressing 3D scalability. \textbf{Airfoil} \citep{23:FourierNO} considers 2D airfoil flow on irregular geometries (e.g., NACA profiles) to test mesh-/shape-generalization. From \textbf{The Well} \citep{24:TheWell}, \textbf{PlanetSWE} is a forced rotating shallow-water system on the sphere (equiangular $256{\times}512$), and \textbf{\texttt{turbulent\_radiative\_layer\_2D}} is a Kelvin--Helmholtz-driven turbulent mixing layer with radiative cooling on a Cartesian $384{\times}128$ grid. \textbf{ERA5} \citep{23:ERA5} uses one month on the $0.25^\circ$ global grid ($721{\times}1440$), with variables 2\,m temperature ($t$) and 10\,m zonal wind ($u$). \textbf{CARRA} \citep{20:CARRA} uses one Arctic month on its native regional grid ($989{\times}789$) with an irregular land/sea/ice mask, and we retain only 10\,m meridional wind ($v_{10}$). We train one-step operators and assess multi-step rollouts on temporal datasets, using native grids/meshes, and apply latitude/area weighting on spherical domains (ERA5, PlanetSWE) and the native mask for CARRA.

\paragraph{Baselines.}
We compare against a broad set of neural operators spanning spectral, convolutional, multiresolution, graph-based, and attention-based designs. 
We include physics-inspired / structure-aware operators: \textbf{(Geo-)FNO} \citep{21:fno,23:FourierNO} (Fourier neural operator on regular grids and its geometry-aware variant), \textbf{CNO} \citep{23:ConvolutionalNO} (convolutional neural operator), \textbf{LSM} \citep{23:LSM} (learned spectral mixing), \textbf{M2NO} \citep{24:M2NO} (multiresolution neural operator), and \textbf{AMG} \citep{25:HarnessingSP} (multi-graph neural operator for arbitrary geometries). 
We also compare with Transformer-based operators that parameterize the operator via attention (global coupling): \textbf{Galerkin Transformer} \citep{21:GT}, \textbf{GNOT} \citep{23:GNOT}, \textbf{ONO} \citep{23:ONO}, \textbf{LNO} \citep{24:LatentNO}, and \textbf{Transolver} \citep{24:Transolver}. 
All baselines use identical data splits, losses, and rollout protocols, and we keep model capacity comparable; 2D/3D/irregular variants are used where applicable.

\paragraph{Implementations.}
We evaluate all models using the relative $L_{2}$ error on held-out sets. Inputs/targets are \emph{normalized per variable} using training statistics; models are trained on the normalized data, and \emph{all metrics are computed after inverse normalization}. Training uses AdamW \citep{19:adamw} with an initial learning rate of $10^{-3}$ and a StepLR scheduler that reduces the learning rate at fixed intervals. All experiments are run on a single NVIDIA RTX~4090 GPU. Per-dataset and per-baseline hyperparameters are provided in Appx.\ref{sec:implementation}.

\begin{figure*}[ht]
    \centering
    \includegraphics[width=\linewidth]{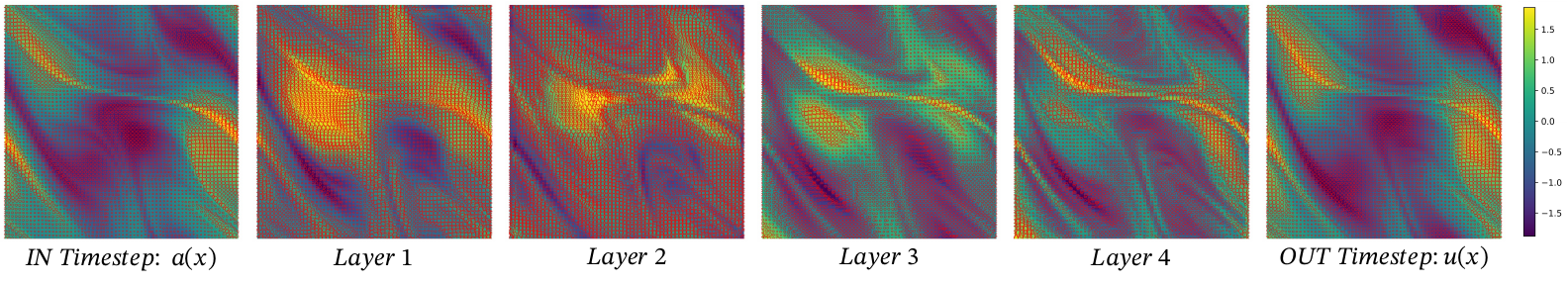}
    \vspace{-20pt}
    \caption{\textbf{Layer-wise evolution of the Gaussian particle field.}
    Particle activations after successive PG Gaussian Attention layers.}
    \label{fig:layerwise-particles}
\end{figure*}

\subsection{Benchmark Performance}
Table~\ref{table:all} reports $L_2$ errors across synthetic Navier--Stokes surrogates, real reanalyses, and additional PDE systems spanning irregular geometries (Airfoil), Cartesian turbulence with radiative effects, and spherical dynamics (PlanetSWE). Overall, the results are consistent with the inductive biases of each model family while demonstrating that GPO is a strong and stable performer across all regimes.

(i) \textbf{Regular-grid fluid surrogates (NS2D/NS3D).} Spectral or multiresolution priors are advantageous on clean periodic grids: Fourier- and multiscale-based baselines are consistently strong, reflecting their efficiency at representing globally coherent modes. GPO matches or surpasses these methods on both 2D and 3D Navier--Stokes, indicating that its Gaussian particle basis can recover the benefits of modal representations without committing to a fixed Fourier parameterization. This supports our motivation that \emph{learned localized atoms} coupled through a \emph{Petrov--Galerkin measurement} can serve as an effective alternative to hand-crafted spectral kernels, while remaining applicable beyond periodic grids.

(ii) \textbf{Large-scale geophysical reanalyses (ERA5/CARRA).} On global and masked regional grids, performance favors models that combine (a) non-Euclidean or mask-aware handling and (b) sufficient long-range coupling. Purely spectral designs degrade under spherical weighting and irregular masks, while attention-style operators improve by modeling global dependencies but may incur unfavorable scaling. GPO achieves the best overall accuracy on these reanalyses, consistent with our contributions: the \emph{Gaussian particle representation} provides localized, interpretable support that adapts naturally to masked domains, and \emph{PG Gaussian Attention} enables global interaction in a modal space, reducing sensitivity to grid irregularity and avoiding quadratic dependence on the number of spatial samples.

(iii) \textbf{Geometry generalization and heterogeneous physics (Airfoil / Turbulent / PlanetSWE).} The additional benchmarks highlight transfer across geometry and operator classes. On \textbf{Airfoil}, which stresses irregular boundaries and shape variation, GPO outperforms geometry-aware spectral and attention baselines, suggesting that particle-based latent states are well-suited to representing boundary-localized structures. On \textbf{Turbulent}, where sharp gradients and multi-scale mixing are dominant, GPO is consistently better than both convolutional and transformer baselines, aligning with our claim that localized Gaussian atoms capture high-frequency structures more effectively than globally parameterized kernels. On \textbf{PlanetSWE}, which requires global coupling on the sphere, GPO also leads, indicating that PG-attention provides sufficient long-range interaction while maintaining robustness across coordinate systems.

\noindent\textbf{Summary.} Across all benchmarks, GPO is competitive on regular synthetic settings and shows clear advantages on irregular, masked, and non-Cartesian domains. This empirical pattern aligns with our central thesis: representing PDE states as \emph{compact, localized, and interpretable Gaussian particles} and evolving them with \emph{modal Petrov--Galerkin attention} yields a general-purpose operator that remains accurate under changing geometry, resolution, and physics.

\subsection{Interpretable Diagnostics}
\label{sec:interpretable}

\subsubsection{Gaussian Reconstruction}
\paragraph{Setup.}
Before evaluating the operator, we validate that the learned \emph{Gaussian particle basis} (Sec.~\ref{sec:gaussian_basis}) is \emph{faithful} and \emph{interpretable}. We train the encoder--decoder to reconstruct the field using weighted Gaussian evaluation, together with particle regularizers (center alignment and a scale-range barrier). For both \textbf{in-distribution (ID)} and \textbf{out-of-distribution (OOD)} samples, we visualize the ground truth, reconstruction, absolute error, and particle overlays (ellipses: center $=\mu$, axes $\propto\sigma$, color/size $\propto w$).

\paragraph{Observations.}
Fig.~\ref{fig:recon} shows that (i) particles consistently concentrate on coherent structures (fronts, filaments, vortices), with \emph{anisotropic} footprints aligned to local flow directions; (ii) reconstruction errors remain localized near sharp gradients and subgrid-scale filaments; and (iii) under OOD shifts, particle geometry and weights adapt smoothly, preserving large- and meso-scale structures with only moderate degradation. This supports that Gaussian particles form a \emph{visualizable, mesh-agnostic} state representation that remains stable beyond the training distribution, and provides an explicit trial space for subsequent operator updates.

\subsubsection{Layer-wise Particle Dynamics}
\label{sec:layerwise-gaussian}
\paragraph{Setup.}
During prediction, we apply the Gaussian Particle Operator (Sec.~\ref{sec:pg-attn}) for $n$ stages. The encoder-fixed particle geometry $(\mu,\sigma)$ defines trial atoms, while PG Gaussian Attention updates per-site coefficients $z^{(k)}_j \in \mathbb{R}^{G}$ at stage $k$. We visualize layer-wise particle fields by overlaying particle footprints and plotting a scalar activation summary (e.g., $A^{(k)}(x_j)=\sum_{g=1}^{G} z^{(k)}_{j,g}$), making the redistribution of modal energy directly observable.

\paragraph{Interpretation.}
In Fig.~\ref{fig:layerwise-particles}, the evolution of activation maps admits a mechanistic reading: (i) local homogenization resembles diffusion among nearby modes; (ii) directional transfers reflect advection-like transport aligned with anisotropic particle geometry; and (iii) localized growth/decay indicates cross-scale energy exchange. Comparing the input $a(x)$ and target $u(x)$, the progressive adjustment from Layer~1$\to$Layer~$n$ tracks the emergence and transport of the fine-scale structures that differentiate $u$ from $a$. This supports our second contribution: PG-attention yields \emph{interpretable global coupling} operating in a compact modal space.

\definecolor{vir0}{RGB}{068,004,090}
\definecolor{vir1}{RGB}{065,062,133}
\definecolor{vir2}{RGB}{048,104,141}
\definecolor{vir3}{RGB}{031,146,139}
\definecolor{vir4}{RGB}{053,183,119}
\definecolor{vir5}{RGB}{145,213,066}
\definecolor{vir6}{RGB}{248,230,032}

\pgfplotsset{compat=1.18}

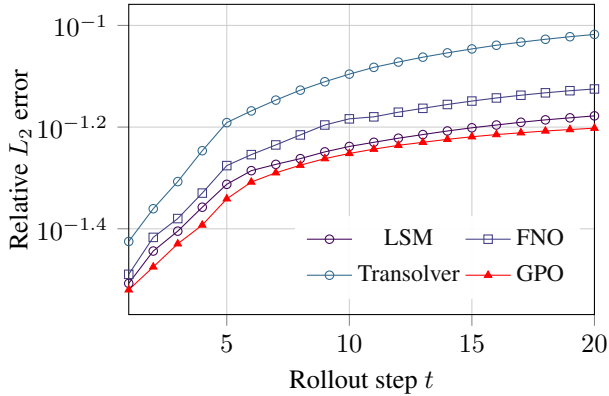
\begin{figure}[t]
  \centering
  \begin{tikzpicture}
    \begin{axis}[
      width=0.75\linewidth,
      height=0.5\linewidth,
      scale only axis,
      xlabel={Rollout step $t$},
      ylabel={Relative $L_2$ error},
      ylabel style={yshift=-0.5em},
      xmin=1, xmax=20,
      ymode=log,
      ymin=2.7e-2, ymax=1.1e-1,
      grid=both,
      grid style={line width=0.1pt, draw=gray!30},
      major grid style={line width=0.2pt, draw=gray!40},
      tick align=outside,
      tick pos=left,
      mark size=1.6pt,
      legend style={
        at={(0.97,0.05)},
        anchor=south east,
        draw=none,
        fill=white,
        fill opacity=0.95,
        font=\small,
        row sep=2pt,
        inner xsep=4pt,
        inner ysep=3pt,
        legend columns=2,
      },
    ]
    
      \addplot+[mark=o, color=vir0] coordinates {
        (1,3.11E-02) (2,3.60E-02) (3,3.94E-02) (4,4.39E-02) (5,4.87E-02)
        (6,5.18E-02) (7,5.33E-02) (8,5.47E-02) (9,5.64E-02) (10,5.78E-02)
        (11,5.89E-02) (12,6.00E-02) (13,6.10E-02) (14,6.20E-02) (15,6.29E-02)
        (16,6.37E-02) (17,6.45E-02) (18,6.52E-02) (19,6.58E-02) (20,6.64E-02)
      };
      \addlegendentry{LSM}

      \addplot+[mark=square, color=vir1] coordinates {
        (1,3.24E-02) (2,3.83E-02) (3,4.17E-02) (4,4.68E-02) (5,5.30E-02)
        (6,5.57E-02) (7,5.82E-02) (8,6.09E-02) (9,6.37E-02) (10,6.55E-02)
        (11,6.61E-02) (12,6.75E-02) (13,6.87E-02) (14,6.99E-02) (15,7.10E-02)
        (16,7.20E-02) (17,7.29E-02) (18,7.37E-02) (19,7.44E-02) (20,7.50E-02)
      };
      \addlegendentry{FNO}

      \addplot+[mark=o, color=vir2] coordinates {
        (1,3.76E-02) (2,4.36E-02) (3,4.93E-02) (4,5.67E-02) (5,6.44E-02)
        (6,6.79E-02) (7,7.13E-02) (8,7.46E-02) (9,7.75E-02) (10,8.02E-02)
        (11,8.27E-02) (12,8.47E-02) (13,8.66E-02) (14,8.83E-02) (15,8.99E-02)
        (16,9.14E-02) (17,9.27E-02) (18,9.39E-02) (19,9.50E-02) (20,9.60E-02)
      };
      \addlegendentry{Transolver}

      \addplot+[mark=triangle*, color=red] coordinates {
        (1,3.02E-02) (2,3.35E-02) (3,3.72E-02) (4,4.05E-02) (5,4.56E-02)
        (6,4.92E-02) (7,5.13E-02) (8,5.31E-02) (9,5.47E-02) (10,5.60E-02)
        (11,5.71E-02) (12,5.81E-02) (13,5.89E-02) (14,5.97E-02) (15,6.04E-02)
        (16,6.10E-02) (17,6.15E-02) (18,6.20E-02) (19,6.24E-02) (20,6.28E-02)
      };
      \addlegendentry{GPO}

    \end{axis}
  \end{tikzpicture}
  \vspace{-20pt}
  \caption{Rollout stability schematic: error versus rollout horizon.}
  \label{fig:rollout_schematic}
\end{figure}

\subsubsection{Rollout Stability}
\paragraph{Setup.}
To probe long-horizon behavior, we perform autoregressive rollouts and report the relative $L_2$ error as a function of rollout step $t$ (Figure~\ref{fig:rollout_schematic}). We compare representative spectral and attention baselines against GPO under identical rollout protocols.

\paragraph{Observations.}
All methods exhibit error accumulation over time, but GPO shows consistently slower growth and reduced drift across the horizon. This behavior is consistent with our design: the particle geometry anchors the representation to localized structures, while PG measurement and modal coupling mitigate compounding aliasing and stabilize global interactions without directly attending over all spatial points. Practically, this indicates that the learned dynamics remain coherent under repeated composition, which is critical for real forecasting workloads.

\subsubsection{Spectral Fidelity}
\paragraph{Setup.}
Beyond pointwise errors, we assess physical plausibility via the kinetic energy spectrum $E(k)$ on NS2D (Fig.~\ref{fig:ns2d_energy_spectrum}), computed from Fourier modes and shown on a log--log scale with a reference inertial-range slope.

\paragraph{Observations.}
Compared to baselines, GPO tracks the ground-truth spectrum more closely across both low and intermediate wavenumbers, and exhibits a delayed high-$k$ roll-off, indicating less over-smoothing of small-scale content. In particular, spectral baselines tend to lose energy early in the high-frequency range, while attention models partially recover mid-range structure but still attenuate high-$k$ modes. The improved spectral alignment of GPO supports our central thesis: representing states with localized Gaussian particles and evolving them through PG-attention better preserves multi-scale structure, yielding dynamics that are not only accurate but also physically consistent. Additional visualizations are provided in Appx.~\ref{sec:vis}.

\subsection{Model Analysis}
\textbf{Ablations} (Table~\ref{table:ablation}) confirm that the \emph{Gaussian Field} is essential, while even a single Gaussian per site improves over MLP baselines. The best results arise from combining the Gaussian Field with \emph{PG Gaussian Attention}. Increasing the modal budget $G$ generally helps but exhibits diminishing returns; we therefore adopt a moderate $G$ for a balanced accuracy–efficiency trade-off. 

\textbf{Complexity} (Table~\ref{table:efficiency}) shows that GPO maintains low memory use and competitive runtime, scaling near-linearly with the number of query points (unlike spatial self-attention), and offering favorable trade-offs across 2D/3D and irregular domains. See Appx.\ref{sec:analysis} for full ablations and analyses.

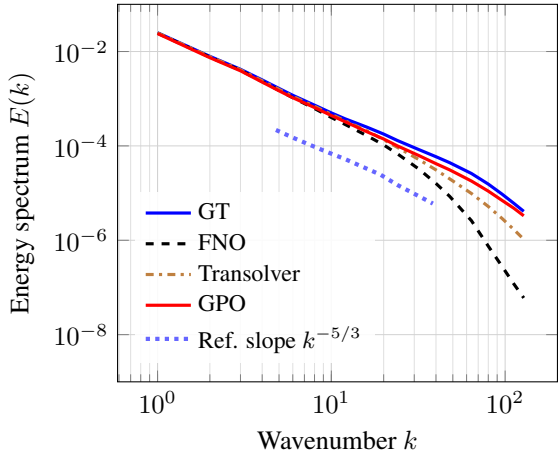
\begin{figure}[t]
  \centering
  \begin{tikzpicture}
    \begin{axis}[
      width=0.9\linewidth,
      height=0.8\linewidth,
      xmode=log,
      ymode=log,
      xlabel={Wavenumber $k$},
      ylabel={Energy spectrum $E(k)$},
      xmin=0, xmax=200,          
      ymin=1e-9, ymax=1e-1,
      grid=both,
      grid style={line width=0.1pt, draw=gray!25},
      major grid style={line width=0.2pt, draw=gray!35},
      legend style={
        at={(0.04,0.05)},
        anchor=south west,
        draw=none,
        fill=white,
        fill opacity=0.9,
        text opacity=1,
        font=\small,
        legend columns=1,
      },
      legend cell align=left,
      mark size=1.5pt,
      set layers,
    ]

      \addplot+[mark=none, very thick, color=blue] coordinates {
        (1,2.50e-2) (2,7.90e-3) (3,4.20e-3) (4,2.50e-3) (5,1.65e-3)
        (6,1.18e-3) (8,7.30e-4) (10,5.00e-4) (12,3.75e-4) (16,2.50e-4)
        (20,1.78e-4) (24,1.30e-4) (32,8.30e-5) (40,5.90e-5) (48,4.40e-5)
        (64,2.60e-5) (80,1.55e-5) (96,9.50e-6) (112,6.10e-6) (128,4.10e-6)
      };
      \addlegendentry{GT}

      \addplot+[mark=none, very thick, color=black, dashed] coordinates {
        (1,2.45e-2) (2,7.60e-3) (3,4.00e-3) (4,2.30e-3) (5,1.50e-3)
        (6,1.05e-3) (8,6.20e-4) (10,4.00e-4) (12,2.85e-4) (16,1.65e-4)
        (20,1.05e-4) (24,6.80e-5) (32,3.20e-5) (40,1.60e-5) (48,8.50e-6)
        (64,2.60e-6) (80,7.50e-7) (96,2.80e-7) (112,1.20e-7) (128,6.00e-8)
      };
      \addlegendentry{FNO}

      \addplot+[mark=none, very thick, color=brown, dashdotted] coordinates {
        (1,2.48e-2) (2,7.80e-3) (3,4.10e-3) (4,2.40e-3) (5,1.58e-3)
        (6,1.12e-3) (8,6.90e-4) (10,4.65e-4) (12,3.40e-4) (16,2.10e-4)
        (20,1.35e-4) (24,9.30e-5) (32,5.10e-5) (40,3.10e-5) (48,2.00e-5)
        (64,9.80e-6) (80,5.20e-6) (96,2.90e-6) (112,1.70e-6) (128,1.05e-6)
      };
      \addlegendentry{Transolver}

      \addplot+[mark=none, very thick, color=red] coordinates {
        (1,2.40e-2) (2,7.40e-3) (3,3.90e-3) (4,2.25e-3) (5,1.47e-3)
        (6,1.05e-3) (8,6.50e-4) (10,4.40e-4) (12,3.25e-4) (16,2.00e-4)
        (20,1.40e-4) (24,1.00e-4) (32,6.20e-5) (40,4.20e-5) (48,3.10e-5)
        (64,1.80e-5) (80,1.10e-5) (96,7.00e-6) (112,4.80e-6) (128,3.30e-6)
      };
      \addlegendentry{GPO}

      \addplot+[mark=none, ultra thick, color=blue!60, dotted] coordinates {
      (4.8,2.175e-4) (6.0,1.50e-4) (7.2,1.125e-4) (9.6,7.25e-5)
      (12.0,5.25e-5) (14.4,3.90e-5) (19.2,2.45e-5) (24.0,1.485e-5) (38.4,6.00e-6)
    };
    \addlegendentry{Ref. slope $k^{-5/3}$}

    \end{axis}
  \end{tikzpicture}
  \vspace{-5pt}
  \caption{NS2D energy spectrum (log--log).}
  \vspace{-10pt}
  \label{fig:ns2d_energy_spectrum}
\end{figure}

\section{Conclusion}

We introduced the \emph{Gaussian Particle Operator} (GPO), a resolution-agnostic neural operator that represents fields with an interpretable \emph{Gaussian particle basis} and performs basis-to-basis updates via \emph{Petrov--Galerkin Gaussian Attention}. The pipeline makes intermediate objects—particles $(\mu,\sigma,w)$, modal windows, and inter-modal couplings—directly visualizable, while remaining robust across regular grids, irregular masks, and 2D$\rightarrow$3D settings. Empirically, GPO is consistently strong on large-scale geophysical datasets, with improved spectral behavior and stable rollouts.

\subsection*{Limitations and future work}

\textbf{Scalability.} Accuracy and cost depend on the modal budget $G$ and head width. While the $N\!\leftrightarrow\!G$ transfers are linear in $N$, storing and applying $NG$ windows can become a bottleneck at extreme resolutions. Promising directions include adaptive particles, sparse routing/pruning, structured or low-rank coupling in modal space.

\textbf{Physics integration.} Training is largely data-driven with lightweight regularization, and does not enforce invariants or hard constraints (e.g., divergence-free flow or boundary conditions). Future work will integrate PDE structure into the particle basis and PG updates, aiming for operators that are both accurate and more \emph{mechanistically} interpretable.


\bibliography{example_paper}

@book{12:fem,
    author = {Claes Johnson},
    title = {Numerical Solution of Partial Differential Equations by the Finite Element Method},
    publisher = {Courier Corporation, North Chelmsford},
    year = {2012}
}

@article{84:diff,
 ISSN = {00361399},
 URL = {http://www.jstor.org/stable/2101307},
 author = {Gene A. Klaasen and William C. Troy},
 journal = {SIAM Journal on Applied Mathematics},
 number = {1},
 pages = {96--110},
 publisher = {Society for Industrial and Applied Mathematics},
 title = {Stationary Wave Solutions of a System of Reaction-Diffusion Equations Derived from the Fitzhugh-Nagumo Equations},
 volume = {44},
 year = {1984}
}

@article{91:UniversalAU,
    author = {Park, J. and Sandberg, I. W.},
    title = {Universal Approximation Using Radial-Basis-Function Networks},
    journal = {Neural Computation},
    volume = {3},
    number = {2},
    pages = {246-257},
    year = {1991},
    month = {06},
    issn = {0899-7667},
    doi = {10.1162/neco.1991.3.2.246},
    url = {https://doi.org/10.1162/neco.1991.3.2.246},
}

@article{00:RadialBF,
  title={Radial basis functions},
  author={Buhmann, Martin Dietrich},
  journal={Acta numerica},
  volume={9},
  pages={1--38},
  year={2000},
  publisher={Cambridge university press}
}

@article{82:StreamlinePG,
title = {Streamline upwind/Petrov-Galerkin formulations for convection dominated flows with particular emphasis on the incompressible Navier-Stokes equations},
author = {Alexander N. Brooks and Thomas J.R. Hughes},
journal = {Computer Methods in Applied Mechanics and Engineering},
volume = {32},
number = {1},
pages = {199-259},
year = {1982},
issn = {0045-7825},
doi = {https://doi.org/10.1016/0045-7825(82)90071-8},
}

@article{06:RevisitingSF,
title = {Revisiting stabilized finite element methods for the advective–diffusive equation},
author = {Leopoldo P. Franca and Guillermo Hauke and Arif Masud},
journal = {Computer Methods in Applied Mechanics and Engineering},
volume = {195},
number = {13},
pages = {1560-1572},
year = {2006},
issn = {0045-7825},
doi = {https://doi.org/10.1016/j.cma.2005.05.028},
url = {https://www.sciencedirect.com/science/article/pii/S0045782505002951},
}

@article{21:DeepONet,
  author       = {Lu Lu and
                  Pengzhan Jin and
                  Guofei Pang and
                  Zhongqiang Zhang and
                  George Em Karniadakis},
  title        = {Learning nonlinear operators via DeepONet based on the universal approximation
                  theorem of operators},
  journal      = {Nat. Mach. Intell.},
  volume       = {3},
  number       = {3},
  pages        = {218--229},
  year         = {2021}
}

@inproceedings{21:GT,
  author       = {Shuhao Cao},
  title        = {Choose a Transformer: Fourier or Galerkin},
  booktitle    = {NeurIPS},
  pages        = {24924--24940},
  year         = {2021}
}

@inproceedings{21:fno,
  author       = {Zongyi Li and
                  Nikola Borislavov Kovachki and
                  Kamyar Azizzadenesheli and
                  Burigede Liu and
                  Kaushik Bhattacharya and
                  Andrew M. Stuart and
                  Anima Anandkumar},
  title        = {Fourier Neural Operator for Parametric Partial Differential Equations},
  booktitle    = {{ICLR}},
  publisher    = {OpenReview.net},
  year         = {2021}
}

@inproceedings{23:LSM,
  author       = {Haixu Wu and
                  Tengge Hu and
                  Huakun Luo and
                  Jianmin Wang and
                  Mingsheng Long},
  title        = {Solving High-Dimensional PDEs with Latent Spectral Models},
  booktitle    = {{ICML}},
  series       = {Proceedings of Machine Learning Research},
  volume       = {202},
  pages        = {37417--37438},
  publisher    = {{PMLR}},
  year         = {2023}
}

@article{23:no,
  author       = {Nikola B. Kovachki and
                  Zongyi Li and
                  Burigede Liu and
                  Kamyar Azizzadenesheli and
                  Kaushik Bhattacharya and
                  Andrew M. Stuart and
                  Anima Anandkumar},
  title        = {Neural Operator: Learning Maps Between Function Spaces With Applications to PDEs},
  journal      = {J. Mach. Learn. Res.},
  volume       = {24},
  pages        = {89:1--89:97},
  year         = {2023}
}

@article{24:Transolver,
  title={Transolver: A Fast Transformer Solver for PDEs on General Geometries},
  author={Haixu Wu and Huakun Luo and Haowen Wang and Jianmin Wang and Mingsheng Long},
  journal={ArXiv},
  year={2024},
  volume={abs/2402.02366},
}

@inproceedings{25:HarnessingSP,
  author       = {Zhihao Li and
                  Haoze Song and
                  Di Xiao and
                  Zhilu Lai and
                  Wei Wang},
  title        = {Harnessing Scale and Physics: {A} Multi-Graph Neural Operator Framework
                  for PDEs on Arbitrary Geometries},
  booktitle    = {{KDD} {(1)}},
  pages        = {729--740},
  publisher    = {{ACM}},
  year         = {2025}
}

@article{23:ONO,
	title = {Improved {Operator} {Learning} by {Orthogonal} {Attention}},
	author = {Xiao, Zipeng and Hao, Zhongkai and Lin, Bokai and Deng, Zhijie and Su, Hang},
        journal={ArXiv},
	year = {2024},
        volume={abs/2310.12487}
}

@article{23:GNOT,
  title={GNOT: A General Neural Operator Transformer for Operator Learning},
  author={Zhongkai Hao and Chengyang Ying and Zhengyi Wang and Hang Su and Yinpeng Dong and Songming Liu and Ze Cheng and Jun Zhu and Jian Song},
  journal={ArXiv},
  year={2023},
  volume={abs/2302.14376},
}

@book{02:pde,
    author = {Abdul-Majid Wazwaz},
    title = {Partial Differential Equations: Methods and Applications},
    publisher = {Balkema Publishers, Leiden},
    year = {2002}
}

@article{87:turb,
  title = {L\'evy dynamics of enhanced diffusion: Application to turbulence},
  author = {Shlesinger, M. F. and West, B. J. and Klafter, J.},
  journal = {Phys. Rev. Lett.},
  volume = {58},
  issue = {11},
  pages = {1100--1103},
  numpages = {0},
  year = {1987},
  month = {Mar},
  publisher = {American Physical Society},
  doi = {10.1103/PhysRevLett.58.1100},
  url = {https://link.aps.org/doi/10.1103/PhysRevLett.58.1100}
}

@article{20:turb,
author = {Ryan McKeown  and Rodolfo Ostilla-Mónico  and Alain Pumir  and Michael P. Brenner  and Shmuel M. Rubinstein },
title = {Turbulence generation through an iterative cascade of the elliptical instability},
journal = {Science Advances},
volume = {6},
number = {9},
pages = {eaaz2717},
year = {2020},
doi = {10.1126/sciadv.aaz2717},
URL = {https://www.science.org/doi/abs/10.1126/sciadv.aaz2717},
}

@book{82:intro,
    author = {Morton E. Gurtin},
    title = {An Introduction to Continuum Mechanics, Mathematics in Science and Engineering},
    publisher = {Academic Press, Cambridge},
    year = {1982}
}

@inproceedings{23:OperatorLN,
  author       = {Louis Serrano and
                  Lise Le Boudec and
                  Armand Kassa{\"{\i}} Koupa{\"{\i}} and
                  Thomas X. Wang and
                  Yuan Yin and
                  Jean{-}No{\"{e}}l Vittaut and
                  Patrick Gallinari},
  title        = {Operator Learning with Neural Fields: Tackling PDEs on General Geometries},
  booktitle    = {NeurIPS},
  year         = {2023}
}

@inproceedings{24:AROMA,
  author       = {Louis Serrano and
                  Thomas X. Wang and
                  Etienne Le Naour and
                  Jean{-}No{\"{e}}l Vittaut and
                  Patrick Gallinari},
  title        = {{AROMA:} Preserving Spatial Structure for Latent {PDE} Modeling with
                  Local Neural Fields},
  booktitle    = {NeurIPS},
  year         = {2024}
}

@inproceedings{21:mwt,
  author       = {Gaurav Gupta and
                  Xiongye Xiao and
                  Paul Bogdan},
  title        = {Multiwavelet-based Operator Learning for Differential Equations},
  booktitle    = {NeurIPS},
  pages        = {24048--24062},
  year         = {2021}
}

@misc{24:M2NO,
      title={M2NO: An Efficient Multi-Resolution Operator Framework for Dynamic Multi-Scale PDE Solvers}, 
      author={Zhihao Li and Zhilu Lai and Xiaobo Zhang and Wei Wang},
      year={2025},
      eprint={2406.04822},
      archivePrefix={arXiv},
      primaryClass={cs.LG},
      url={https://arxiv.org/abs/2406.04822}, 
}

@article{24:MgNO,
  title={MgNO: Efficient parameterization of linear operators via multigrid},
  author={He, Juncai and Liu, Xinliang and Xu, Jinchao},
  journal={ArXiv},
  year={2024},
  volume={abs/2310.19809}
}

@inproceedings{20:MGNO,
  author       = {Zongyi Li and
                  Nikola B. Kovachki and
                  Kamyar Azizzadenesheli and
                  Burigede Liu and
                  Andrew M. Stuart and
                  Kaushik Bhattacharya and
                  Anima Anandkumar},
  title        = {Multipole Graph Neural Operator for Parametric Partial Differential
                  Equations},
  booktitle    = {NeurIPS},
  year         = {2020}
}

@article{20:GNO,
  title={Neural Operator: Graph Kernel Network for Partial Differential Equations},
  author={Zong-Yi Li and Nikola B. Kovachki and Kamyar Azizzadenesheli and Burigede Liu and Kaushik Bhattacharya and Andrew M. Stuart and Anima Anandkumar},
  journal={ArXiv},
  year={2020},
  volume={abs/2003.03485},
}

@inproceedings{22:bench,
  author       = {Makoto Takamoto and
                  Timothy Praditia and
                  Raphael Leiteritz and
                  Daniel MacKinlay and
                  Francesco Alesiani and
                  Dirk Pfl{\"{u}}ger and
                  Mathias Niepert},
  title        = {PDEBench: An Extensive Benchmark for Scientific Machine Learning},
  booktitle    = {NeurIPS},
  year         = {2022}
}

@misc{23:ERA5,
  author    = {Hersbach, H. and Bell, B. and Berrisford, P. and Biavati, G. and Hor{\'a}nyi, A. and Mu{\~n}oz Sabater, J. and Nicolas, J. and Peubey, C. and Radu, R. and Rozum, I. and Schepers, D. and Simmons, A. and Soci, C. and Dee, D. and Th{\'e}paut, J.-N.},
  title     = {ERA5 hourly data on single levels from 1940 to present},
  year      = {2023},
  publisher = {Copernicus Climate Change Service (C3S) Climate Data Store (CDS)},
  doi       = {10.24381/cds.adbb2d47},
  url       = {https://doi.org/10.24381/cds.adbb2d47},
}

@misc{20:CARRA,
  author = {Schyberg, H. and Yang, X. and K{\o}ltzow, M.~A.~{\O}. and Amstrup, B. and Bakketun, {\AA}. and Bazile, E. and Bojarova, J. and Box, J.~E. and Dahlgren, P. and Hagelin, S. and Homleid, M. and Hor{\'a}nyi, A. and H{\o}yer, J. and Johansson, {\AA}. and Killie, M.~A. and K{\"o}rnich, H. and Le Moigne, P. and Lindskog, M. and Manninen, T. and Nielsen Englyst, P. and Nielsen, K.~P. and Olsson, E. and Palmason, B. and Peralta Aros, C. and Randriamampianina, R. and Samuelsson, P. and Stappers, R. and St{\o}ylen, E. and Thorsteinsson, S. and Valkonen, T. and Wang, Z.~Q.},
  title = {Arctic regional reanalysis on single levels from 1991 to present},
  year = {2020},
  publisher = {Copernicus Climate Change Service (C3S) Climate Data Store (CDS)},
  doi = {10.24381/cds.713858f6},
  url = {https://doi.org/10.24381/cds.713858f6},
}

@inproceedings{19:adamw,
  author       = {Ilya Loshchilov and
                  Frank Hutter},
  title        = {Decoupled Weight Decay Regularization},
  booktitle    = {{ICLR} (Poster)},
  publisher    = {OpenReview.net},
  year         = {2019}
}

@article{23:FourierNO,
  author       = {Zongyi Li and
                  Daniel Zhengyu Huang and
                  Burigede Liu and
                  Anima Anandkumar},
  title        = {Fourier Neural Operator with Learned Deformations for PDEs on General
                  Geometries},
  journal      = {J. Mach. Learn. Res.},
  volume       = {24},
  pages        = {388:1--388:26},
  year         = {2023}
}

@inproceedings{24:TheWell,
  author       = {Ruben Ohana and
                  Michael McCabe and
                  Lucas Meyer and
                  Rudy Morel and
                  Fruzsina J. Agocs and
                  Miguel Beneitez and
                  Marsha Berger and
                  Blakesley Burkhart and
                  Stuart B. Dalziel and
                  Drummond B. Fielding and
                  Daniel Fortunato and
                  Jared A. Goldberg and
                  Keiya Hirashima and
                  Yan{-}Fei Jiang and
                  Rich R. Kerswell and
                  Suryanarayana Maddu and
                  Jonah Miller and
                  Payel Mukhopadhyay and
                  Stefan S. Nixon and
                  Jeff Shen and
                  Romain Watteaux and
                  Bruno R{\'{e}}galdo{-}Saint Blancard and
                  Fran{\c{c}}ois Rozet and
                  Liam Holden Parker and
                  Miles D. Cranmer and
                  Shirley Ho},
  title        = {The Well: a Large-Scale Collection of Diverse Physics Simulations
                  for Machine Learning},
  booktitle    = {NeurIPS},
  year         = {2024}
}

@inproceedings{24:LatentNO,
  author       = {Tian Wang and
                  Chuang Wang},
  title        = {Latent Neural Operator for Solving Forward and Inverse {PDE} Problems},
  booktitle    = {NeurIPS},
  year         = {2024}
}

@inproceedings{23:ConvolutionalNO,
  author       = {Bogdan Raonic and
                  Roberto Molinaro and
                  Tim De Ryck and
                  Tobias Rohner and
                  Francesca Bartolucci and
                  Rima Alaifari and
                  Siddhartha Mishra and
                  Emmanuel de B{\'{e}}zenac},
  title        = {Convolutional Neural Operators for robust and accurate learning of
                  PDEs},
  booktitle    = {NeurIPS},
  year         = {2023}
}
\bibliographystyle{icml2026}

\newpage
\appendix
\onecolumn
\section{Related Work}

\paragraph{Neural operators.}
Classical neural operator methods aim to learn mappings between function spaces directly from data, typically by parameterizing a resolution-agnostic kernel or lifting to a latent space and learning integral transforms. Representative approaches include the \emph{Fourier Neural Operator} (FNO), which performs global convolution via spectral multipliers to approximate operator kernels in Fourier space \citep{21:fno,23:no}, and \emph{DeepONet}, which decomposes an operator into branch/trunk networks to separately encode input functions and query coordinates \citep{21:DeepONet}. Variants extend these ideas with multiresolution bases \citep{20:MGNO,21:mwt,24:MgNO,24:M2NO}, graph or kernelized message passing \citep{25:HarnessingSP}, and learned Green’s functions \citep{20:GNO}.

These models are \emph{purely data driven}: while many designs are physics-inspired, their internal representations are typically opaque. In particular, the learned latent bases and mixing weights are not tied to interpretable physical primitives, which limits diagnostic insight and the ability to attribute predictions to physically meaningful components.

\paragraph{Transformer-based methods.}
A recent line of work adopts Transformers to parameterize neural operators, replacing hand-crafted kernel parameterizations with data-driven attention. Examples include \emph{Galerkin Transformers}, which align attention with variational forms \citep{21:GT}, \emph{GNOT} \citep{23:GNOT} that leverage attention for long-range coupling, \emph{Transolver} \citep{24:Transolver}, which introduces slice-based attention for efficient global mixing, and operator networks that stack attention with physics-informed objectives \citep{23:ONO}. Empirically, with sufficiently large training corpora and careful scaling, attention-based operators often match or surpass traditional neural operators in expressive power and generalization to out-of-distribution forcings and grids.

However, these gains come with two well-known limitations. \textbf{(i) Lack of interpretability:} standard attention weights are not anchored to physically interpretable trial/test functions, making it difficult to ascribe predictions to identifiable modes or localized mechanisms. \textbf{(ii) Frequency bias:} global self-attention tends to emphasize low-rank, global correlations (low-frequency structure), while recovering sharp, localized, or high-frequency phenomena often requires architectural add-ons or extensive data augmentation. As a result, pure Transformer operators provide limited physical attribution and may under-represent fine-scale features without additional inductive biases.

\section{Expressivity of the Gaussian Field and GPO}
\label{sec:expressivity}

\subsection{Expressivity of the Gaussian Field}
\label{sec:exp_gf}

\begin{lemma}[Density of Gaussian mixtures]
\label{lem:gm-density}
On compact $\Omega$, finite mixtures of anisotropic Gaussians are dense in $C(\Omega)$ (and dense in
$L^r(\Omega)$ for $1\!\le\!r\!<\!\infty$). Hence for any continuous scalar field $v$ and $\varepsilon>0$,
there exist $G,\{\mu_i,\sigma_i,w_i\}_{i=1}^G$ such that
$\|v(\cdot)-\sum_{i=1}^G w_i \exp(-\tfrac{1}{2}\|(\cdot-\mu_i)/\sigma_i\|^2)\|_\infty<\varepsilon$.
\end{lemma}
\textit{Sketch.} Standard universal approximation results for radial basis functions/Gaussian kernels.

\begin{proof}
We give a constructive proof based on Gaussian mollification and Riemann sums.

\paragraph{Step 1: Approximate identity via Gaussian mollifiers.}
Let $\Omega\subset\mathbb{R}^d$ be compact and let $v\in C(\Omega)$.
By Tietze’s extension theorem there exists $\tilde v\in C_c(\mathbb{R}^d)$ such that
$\tilde v|_{\Omega}=v$.
For $\Sigma\in\mathbb{R}^{d\times d}$ symmetric positive definite, set the (unnormalized) Gaussian
\[
\phi_{\Sigma}(x) \;=\; \exp\!\Big(-\tfrac{1}{2}\,x^\top\Sigma^{-1}x\Big).
\]
Let $\{\Sigma_\epsilon\}_{\epsilon\downarrow 0}$ be any family with $\|\Sigma_\epsilon\|\to 0$.
Since Gaussians form an approximate identity, the \emph{normalized} mollification
$\tilde v * \phi_{\Sigma_\epsilon} / \int_{\mathbb{R}^d}\phi_{\Sigma_\epsilon}$ converges to $\tilde v$
\emph{uniformly} on compact sets as $\epsilon\downarrow 0$ (uniform continuity of $\tilde v$ and standard approximate-identity properties). Because the normalization constant is a positive scalar depending only on $\Sigma_\epsilon$, we can absorb it into the mixture weights later. Hence, for any $\eta>0$ there exists $\epsilon_0$ such that for all $0<\epsilon\le\epsilon_0$,
\begin{equation}
\label{eq:unif-moll}
\sup_{x\in\Omega}\Big|\,(\tilde v * \phi_{\Sigma_\epsilon})(x) - \tilde v(x)\,\Big| < \tfrac{\eta}{2}.
\end{equation}

\paragraph{Step 2: Riemann-sum approximation of the convolution (finite mixture).}
Fix such an $\epsilon$, write $\Sigma=\Sigma_\epsilon$, and denote the convolution
\[
(\tilde v * \phi_{\Sigma})(x) \;=\; \int_{\mathbb{R}^d} \tilde v(y)\,\phi_{\Sigma}(x-y)\,dy.
\]
Since $\tilde v$ is compactly supported and continuous while $\phi_\Sigma$ is continuous and rapidly decaying, the integrand is continuous with compact support in $y$ uniformly in $x\in\Omega$. Hence Riemann sums approximate the integral uniformly in $x$:
there exists a finite set of nodes $\{\mu_i\}_{i=1}^G\subset\mathbb{R}^d$ with associated positive quadrature weights $\{\Delta_i\}_{i=1}^G$ such that
\begin{equation}
\label{eq:riemann}
\sup_{x\in\Omega}\left|
(\tilde v * \phi_{\Sigma})(x)\;-\;\sum_{i=1}^{G} \tilde v(\mu_i)\,\phi_{\Sigma}(x-\mu_i)\,\Delta_i
\right| < \tfrac{\eta}{2}.
\end{equation}
Define mixture weights $w_i:=\tilde v(\mu_i)\,\Delta_i$ (real-valued; the lemma does not restrict their sign), and note that each term is exactly a (shared-covariance) Gaussian atom
$\exp\!\big(-\tfrac{1}{2}\|(x-\mu_i)\|_{\Sigma^{-1}}^2\big)$, i.e.,
\[
\sum_{i=1}^{G} w_i\,\phi_{\Sigma}(x-\mu_i)
\;=\;
\sum_{i=1}^{G} w_i \exp\!\Big(-\tfrac{1}{2}(x-\mu_i)^\top\Sigma^{-1}(x-\mu_i)\Big).
\]

\paragraph{Step 3: Uniform approximation on $\Omega$.}
Combining \eqref{eq:unif-moll} and \eqref{eq:riemann},
\begin{align*}
    &\sup_{x\in\Omega}\left|\tilde v(x)-\sum_{i=1}^{G} w_i \exp\!\Big(-\tfrac{1}{2}(x-\mu_i)^\top\Sigma^{-1}(x-\mu_i)\Big)\right| \\
\le&\sup_{x\in\Omega}\big|\tilde v(x)-(\tilde v * \phi_{\Sigma})(x)\big|
+\sup_{x\in\Omega}\big|(\tilde v * \phi_{\Sigma})(x)-\textstyle\sum_i w_i\phi_{\Sigma}(x-\mu_i)\big|\\
<& \eta.
\end{align*}

Restricting back to $\Omega$ (where $\tilde v=v$) yields
\[
\big\|v(\cdot)-\sum_{i=1}^G w_i \exp\!\big(-\tfrac{1}{2}\|(\cdot-\mu_i)\|_{\Sigma^{-1}}^2\big)\big\|_{\infty}
< \eta.
\]
Since $\eta>0$ was arbitrary, finite mixtures of (possibly anisotropic) Gaussians are dense in $C(\Omega)$.

\paragraph{Anisotropy and vector-valued extension.}
We used a common covariance $\Sigma$ for clarity; allowing mode-dependent $\Sigma_i$ only increases expressivity, so the same result holds with per-atom anisotropy. For vector-valued $v$, apply the scalar result componentwise.

\paragraph{$L^r$ density.}
Because $C(\Omega)$ is dense in $L^r(\Omega)$ for $1\le r<\infty$ on compact $\Omega$, the uniform approximation implies $L^r$ approximation, completing the proof.
\end{proof}

\subsection{Expressivity of GPO}
\label{sec:exp_gpo}

\begin{theorem}[Universal approximation in modal form]
\label{thm:universality}
Let $\mathcal{T}:\mathcal{X}\to\mathcal{Y}$ be a continuous operator on compacta that admits a
Hilbert--Schmidt (Mercer-type) kernel $K(\mathbf{x},\mathbf{x}')$ or, more generally, a low-rank
factorization $\mathcal{T}\approx \Phi(\cdot)\,\mathcal{K}\,\Phi(\cdot)^\top$ with continuous
features $\Phi:\Omega\to\mathbb{R}^{m}$. Then, for any $\varepsilon>0$, there exist $G$ and network
parameters $\Theta$ such that $\|\mathcal{G}_\Theta-\mathcal{T}\|_{\mathcal{X}\to\mathcal{Y}}
<\varepsilon$. 
\end{theorem}
\textit{Sketch.} By Lemma~\ref{lem:gm-density} and universal approximation of MLPs, windows
$p(\mathbf{x},g)$ and latent features $S(Z(\mathbf{x}))$ approximate $\Phi(\mathbf{x})$; attention
realizes a trainable $\mathcal{K}$ on the $G$ modes. The scatter-and-decoder emulate the output
feature map. Increasing $G$ and widths yields density in the space of continuous operators.

\begin{proof}
We prove the claim for operators on compact domains by reducing to a finite–rank Mercer
approximation and showing that each stage of our pipeline can approximate the corresponding
finite–dimensional objects arbitrarily well. Throughout, $\|\cdot\|_{\mathcal{X}\to\mathcal{Y}}$
denotes the operator norm on bounded subsets.

\paragraph{Step 0: Mercer (or low–rank) truncation.}
Assume $\mathcal{T}$ is continuous on bounded sets and admits either a Hilbert--Schmidt kernel
$K(\mathbf{x},\mathbf{x}')$ or, more generally, a low–rank factorization
$\mathcal{T}\approx \Phi(\cdot)\,\mathcal{K}\,\Phi(\cdot)^\top$ with continuous
$\Phi:\Omega\to\mathbb{R}^{m}$. In the Mercer case, by spectral theory,
\[
K(\mathbf{x},\mathbf{x}')=\sum_{r=1}^{\infty}\lambda_r\,\varphi_r(\mathbf{x})\,\varphi_r(\mathbf{x}'),
\quad \lambda_r\ge 0,\;\{\varphi_r\}\subset C(\Omega),
\]
and the partial sums define finite–rank operators
$\mathcal{T}_m f(\mathbf{x})=\sum_{r=1}^m \lambda_r \varphi_r(\mathbf{x})\int \varphi_r(\mathbf{x}')f(\mathbf{x}')\,d\mathbf{x}'$
with $\|\mathcal{T}-\mathcal{T}_m\|\to 0$ as $m\to\infty$ (uniform on compacta).
In the given low–rank form, select $m$ and continuous $\Phi_m:\Omega\to\mathbb{R}^m$,
$\mathcal{K}_m\in\mathbb{R}^{m\times m}$ such that
\begin{equation}
\label{eq:rank-m}
\big\|\mathcal{T}-\mathcal{T}_m\big\| < \varepsilon/3,
\qquad
\mathcal{T}_m f(\mathbf{x})
= \Phi_m(\mathbf{x})\,\mathcal{K}_m \int_{\Omega}\Phi_m(\mathbf{x}')^\top f(\mathbf{x}')\,d\mathbf{x}'.
\end{equation}

\paragraph{Step 1: Approximating the feature maps by Gaussian basis + MLPs.}
By Lemma~\ref{lem:gm-density} (density of Gaussian mixtures) and universal approximation of MLPs,
for any $\delta>0$ there exist:
(i) a pointwise encoder/evaluator producing $Z(\mathbf{x})\in\mathbb{R}^{G}$
from Gaussian particles $(\mu,\sigma,w)$ and a small MLP $S$ such that the
\emph{trial features} $\Psi(\mathbf{x})\in\mathbb{R}^{D}$, defined by $\Psi(\mathbf{x})=S(Z(\mathbf{x}))$,
satisfy
\begin{equation}
\label{eq:trial-approx}
\sup_{\mathbf{x}\in\Omega}\big\|\Psi(\mathbf{x})- \Phi_m(\mathbf{x})\big\|_2 < \delta;
\end{equation}
(ii) head–wise \emph{Gaussian modal windows} $p(\mathbf{x},g)\ge 0$ with $\sum_{g=1}^G p(\mathbf{x},g)=1$,
implemented by linear maps on $[Z(\mathbf{x}),(\mu,\sigma,w)(\mathbf{x})]$ and a softmax,
such that the \emph{test} functionals
\[
\mathcal{M}_g(f)=\int_{\Omega} p(\mathbf{x},g)\, f(\mathbf{x})\, d\mathbf{x}
\]
approximate the $m$ target coordinates $\int \Phi_m(\mathbf{x})^\top f(\mathbf{x})\,d\mathbf{x}$
after a fixed linear readout. Concretely, there exists $W\in\mathbb{R}^{m\times G}$ with
$\|\,W\,[\mathcal{M}_g(\cdot)]_{g=1}^G - \int \Phi_m(\cdot)^\top(\cdot)\,\| < C_1\delta$.
(One can view $Wp(\mathbf{x},\cdot)$ as a learned quadrature/test family for the $m$ coordinates.)

\paragraph{Step 2: Discrete PG measurement and quadrature error.}
Given a discretization $\{\mathbf{x}_j\}_{j=1}^{N}$ with empirical measure converging to the
sampling measure on $\Omega$, the $N\!\to\!G$ aggregation used in Sec.~\ref{sec:pg-attn}
forms tokens
\[
t_g \;=\; \frac{\sum_{j=1}^{N} p(\mathbf{x}_j,g)\,\Psi(\mathbf{x}_j)}{\sum_{j=1}^{N} p(\mathbf{x}_j,g)}
\quad\in\mathbb{R}^{D}.
\]
By uniform continuity of $\Psi$ and $p(\cdot,g)$ on compact $\Omega$, Riemann (or Monte Carlo)
sums converge to the integrals. Hence there exists $N_0$ so that for all $N\ge N_0$,
\begin{equation}
\label{eq:meas-approx}
\left\|
\Big[\;t_g\;\Big]_{g=1}^{G} \;-\;
\left[\;\frac{\int p(\mathbf{x},g)\,\Psi(\mathbf{x})\,d\mathbf{x}}
            {\int p(\mathbf{x},g)\,d\mathbf{x}}\;\right]_{g=1}^{G}
\right\|
< C_2\delta .
\end{equation}
Post-multiplying by $W$ and using \eqref{eq:trial-approx} shows that
the vector of $m$ measured coordinates is within $C_3\delta$ of
$\int \Phi_m(\mathbf{x})^\top f(\mathbf{x})\,d\mathbf{x}$ for any $f$ in a bounded set.

\paragraph{Step 3: Implementing the modal coupling by attention + linear maps.}
We next show that the $G\times G$ \emph{modal attention} stage can realize the finite linear map
$\mathcal{K}_m$ (up to basis changes) to arbitrary precision.
Using the head projections $W_z$ and $W_{\mathrm{out}}$, the attention block computes
\[
\widetilde{T} \;=\; \alpha\, \big(T W_V\big),\qquad
Y \;=\; \big(\widetilde{T}\big) W_{\mathrm{out}},
\]
where $T\in\mathbb{R}^{G\times D}$ stacks the tokens $t_g$, $\alpha$ is the
softmax attention matrix, and $W_V,W_{\mathrm{out}}$ are learned linear maps.
Since softmax can approximate a Kronecker–delta (by sending on–diagonal logits to $+\infty$ and
off–diagonal to $-\infty$), we can set $\alpha\approx I_G$ arbitrarily closely. Then
$Y \approx T (W_V W_{\mathrm{out}})$. Because $W_V,W_{\mathrm{out}}$ are unconstrained,
their product can approximate any target matrix $M\in\mathbb{R}^{D\times m}$ to arbitrary precision.
Choosing $M$ to implement the composition $W\,\mathcal{K}_m$ (after the measurement map from Step~2),
we obtain a block that emulates $v \mapsto \mathcal{K}_m v$ in the $m$–dimensional modal coordinates.
(If desired, one may keep $\alpha$ nontrivial and absorb its effect into the surrounding linear maps;
the argument is unchanged.)

\paragraph{Step 4: Scatter and pointwise decoding.}
The $G\!\to\!N$ scatter re-distributes the mixed modal features back to locations via the same
windows $p(\mathbf{x},g)$, followed by a pointwise decoder MLP
$f^{\mathrm{dec}}_{\phi}:\mathbb{R}^{G}\to\mathbb{R}^{c_{\mathrm{out}}}$.
Since MLPs are universal approximators on compacta, the composition can approximate the desired
output feature map $\mathbf{x}\mapsto \Phi_m(\mathbf{x})$ (or its linear image) uniformly, matching
the form in \eqref{eq:rank-m}.

\paragraph{Step 5: Error aggregation.}
Let $\varepsilon_m=\|\mathcal{T}-\mathcal{T}_m\| < \varepsilon/3$ be the truncation error.
Pick $\delta>0$ sufficiently small and $N$ sufficiently large so that:
(i) the feature/window approximations introduce at most $C\delta$ error in the measured
coordinates (Steps 1–2), (ii) the attention+linear block approximates the modal coupling
$\mathcal{K}_m$ within $C\delta$ uniformly on bounded sets (Step 3), and
(iii) the scatter+decoder approximates the output features within $C\delta$ uniformly (Step 4).
By stability (continuity) of all stages,
\[
\big\|\mathcal{G}_\Theta - \mathcal{T}\big\|
\;\le\; \underbrace{\big\|\mathcal{G}_\Theta - \mathcal{T}_m\big\|}_{\le C\delta}
\;+\; \underbrace{\big\|\mathcal{T}_m - \mathcal{T}\big\|}_{\varepsilon_m}
\;<\; C\delta + \varepsilon/3.
\]
Choosing $\delta$ so that $C\delta<2\varepsilon/3$ yields
$\|\mathcal{G}_\Theta-\mathcal{T}\|<\varepsilon$.

Combining the steps completes the proof.
\end{proof}

\section{Implementation Details}
\label{sec:implementation}

\subsection{Baseline implementations}
Baseline models (Geo-FNO, LSM, Galerkin Transformer, GNOT, ONO, Transolver) are adapted from the
\emph{Neural-Solver-Library} \citep{24:Transolver} reference implementation at \url{https://github.com/thuml/Neural-Solver-Library}. Other models are adapted from their official repositories. Unless otherwise noted, we keep an identical training schedule across baselines: AdamW optimizer,
initial learning rate $10^{-3}$ with a \texttt{StepLR} scheduler (step size and decay factor as in the
library’s default per dataset), up to 500 epochs with validation early stopping, the same data
normalization/inverse-normalization protocol, and matched rollout/evaluation settings. 

\subsection{GPO configurations}

The dataset-specific configurations of GPO are summarized in Table~\ref{table:gpo_config}. We provide the source code of GPO in the Supplementary Material.

\begin{table*}[ht]
    \caption{\textbf{Model configurations of GPO.}}
    \label{table:gpo_config}
    \begin{sc}
        \renewcommand{\multirowsetup}{\centering}
        \footnotesize 
        \resizebox{\linewidth}{!}{
        \begin{tabular}{c|cccc}
            \toprule
                \multirow{2}{*}{Benchmarks} & \multicolumn{4}{c}{Model Configurations} \\ 
                \cmidrule{2-5} & hidden\_dim & num\_layers & num\_heads & num\_gaussians \\
            \midrule
                NS2D & 128 & 8 & 8 & 32 \\
                NS3D & 64 & 8 & 4 & 16 \\
                ERA5-temp & 64 & 4 & 4 & 16 \\
                ERA5-wind u & 64 & 4 & 4 & 16 \\
                Carra & 64 & 4 & 4 & 16 \\
                Airfoil   & 64  & 4 & 4 & 16 \\
                Turbulent & 128 & 8 & 8 & 32 \\
                PlanetSWE & 64  & 8 & 4 & 16 \\
            \bottomrule
        \end{tabular}}
    \end{sc}
\end{table*}

\section{Model Analysis}
\label{sec:analysis}

\subsection{Ablation Study}
\label{sec:ablation}

Table~\ref{table:ablation} reports the $L_2$ error under controlled variants (parameter counts are adjusted to be comparable). \textbf{(i) Necessity of the Gaussian Field.} Replacing the Gaussian Field with plain MLP encoder/decoder (\texttt{w/o Gaussian Field}) degrades accuracy markedly ($7.44{\times}10^{-2}$), and removing the PG operator while keeping the Gaussian Field (\texttt{w/o PG Operator}) is even worse ($8.57{\times}10^{-2}$). Notably, even a \emph{single} Gaussian per site (\texttt{num\_gaussian=1}) already improves to $6.28{\times}10^{-2}$, indicating that particleized Gaussian evaluation is a beneficial inductive bias beyond a black-box MLP. \textbf{(ii) Synergy of PG Operator and Gaussian Field.} Combining the Gaussian basis with the PG Gaussian Attention yields the full GPO (baseline: $3.90{\times}10^{-2}$), demonstrating that the PG measurement$\rightarrow$modal coupling$\rightarrow$scatter complements the local particle representation; each component alone is insufficient. \textbf{(iii) Effect of the number of Gaussians.} Increasing \texttt{num\_gaussian} consistently reduces error (from $4.21{\times}10^{-2}$ at $G{=}16$ to $3.84{\times}10^{-2}$ at $G{=}64$), but with diminishing returns; considering cost (Sec.~\ref{sec:complexity}), we adopt $G{=}16/32$ as a practical trade-off between efficiency and accuracy.

\begin{table}[ht]
\centering
\caption{Ablation results comparing the \(L_2\) error of different configurations.}
\label{table:ablation}
\begin{sc}
\resizebox{0.5\linewidth}{!}{
\begin{tabular}{l|l}
\toprule
Model Configuration & $L_{2}$ Error \\ 
\midrule
w/o PG Operator & 8.57E-02 \\
w/o Gaussian Field & 7.44E-02 \\
num\_gaussian = 1 & 6.28E-02 \\
num\_gaussian = 16 & 4.21E-02\\
num\_gaussian = 64 & 3.84E-02\\
\midrule
GPO (Baseline) & 3.90E-02 \\ 
\bottomrule
\end{tabular}}
\end{sc}
\end{table}

\subsection{Computational Complexity}
\label{sec:complexity}

Empirical measurements (Table~\ref{table:efficiency}, $64{\times}64{\times}3$, batch~16) corroborate the analysis:
GPO attains low memory footprint (2{,}313\,MiB) and competitive time (44.66\,s/epoch train; 1.67\,s/epoch
inference) with a modest parameter count (6.10\,MB), outperforming attention baselines in training speed
(Galerkin/Transolver/ONO/GNOT) and GPU memory, while remaining close to spectral baselines at inference.
Although FNO is fastest on this small grid, GPO’s cost grows near–linearly with $N$ and remains stable when
moving to higher resolutions or 3D, where spatial attention becomes prohibitive and FFT memory/IO costs rise.

By aggregating \emph{locally} ($N\!\leftrightarrow\!G$) and coupling \emph{globally} only in modal space
($G{\times}G$), GPO delivers resolution–agnostic efficiency: linear scaling in $N$, controllable quadratic
dependence on $G$, and favorable memory/time trade–offs across 2D/3D and irregular domains.

\begin{table*}[ht]
\centering
\caption{\textbf{Computational efficiency comparison across models} (measured with input size 64$\times$64$\times$3, batch size 16).}
\label{table:efficiency}
\begin{sc}
\resizebox{\linewidth}{!}{
\begin{tabular}{l|cc|ccc}
\toprule
Model & Param Count & Param (MB) & GPU Mem (MiB) & Train (s/epoch) & Inference (s/epoch) \\
\midrule
FNO                     & 640,305       & 4.84  & 949   & 28.27     & 0.5  \\
LSM                     & 19,187,457    & 73.23 & 2,875 & 48.42     & 1.73 \\
Galerkin Transformer    & 1,096,321     & 4.18  & 4,301 & 65.29     & 2.96 \\
GNOT                    & 2,485,901     & 9.48  & 8,643 & 139.42    & 6.09 \\
Transolver              & 3,069,889     & 11.71 & 4,917 & 97.03     & 4.10 \\
ONO                     & 1,596,673     & 6.09  & 6,163 & 94.80     & 4.27 \\
\midrule
\textbf{GPO (Ours)}     & 1,598,257     & 6.10  & 2,313 & 44.66     & 1.67 \\
\bottomrule
\end{tabular}}
\end{sc}
\end{table*}

\section{Additional Visualizations}
\label{sec:vis}

\begin{figure}[ht]
    \centering
    \includegraphics[width=0.85\linewidth]{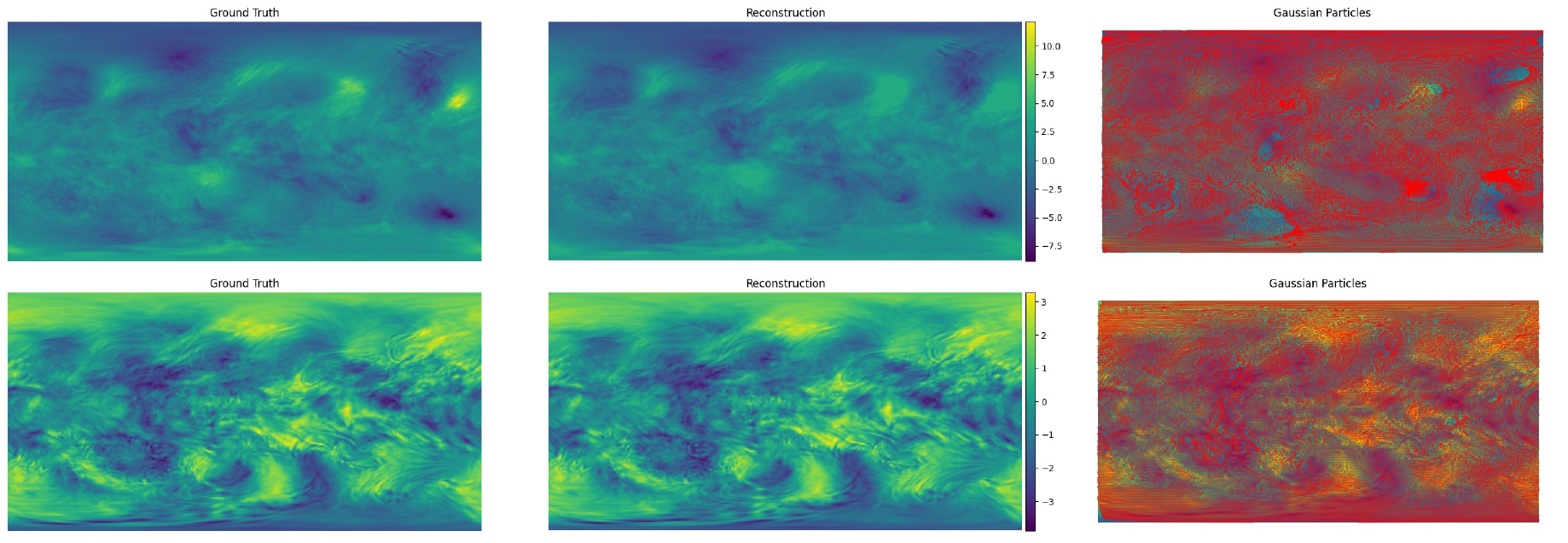}
    \caption{Interpretable visualization of ERA5 on an \textbf{in-distribution sample (above)} and an \textbf{out-of-distribution sample (below)}. Left to right: ground truth, reconstruction from the Gaussian basis and learned Gaussian particles overlaid (ellipses: center $\mu$, axes $\propto\sigma$, color/size $\propto w$). }
    \label{fig:era5}
\end{figure}

\end{document}